\documentclass[11pt,reqno,a4paper]{amsart}
\usepackage[toc,page]{appendix}

\usepackage[utf8]{inputenc}
\usepackage{esint}
\usepackage{MnSymbol}
\usepackage{enumitem}
\usepackage[english]{babel}
\usepackage{verbatim}

\usepackage[
	colorlinks,
	pdfpagelabels,
	pdfstartview = FitH,
	bookmarksopen = true,
	bookmarksnumbered = true,
	linkcolor = blue,
	plainpages = false,
	hypertexnames = false,
	citecolor = red] {hyperref}

\hypersetup{
	linktoc=page
}

\usepackage[capitalize]{cleveref}
\crefname{enumi}{}{}
\crefname{equation}{}{}

\makeatletter
\def\@tocline#1#2#3#4#5#6#7{\relax
  \ifnum #1>\c@tocdepth 
  \else
    \par \addpenalty\@secpenalty\addvspace{#2}%
    \begingroup \hyphenpenalty\@M
    \@ifempty{#4}{%
      \@tempdima\csname r@tocindent\number#1\endcsname\relax
    }{%
      \@tempdima#4\relax
    }%
    \parindent\z@ \leftskip#3\relax \advance\leftskip\@tempdima\relax
    \rightskip\@pnumwidth plus4em \parfillskip-\@pnumwidth
    #5\leavevmode\hskip-\@tempdima
      \ifcase #1
       \or\or \hskip 1em \or \hskip 2em \else \hskip 3em \fi%
      #6\nobreak\relax
    \dotfill\hbox to\@pnumwidth{\@tocpagenum{#7}}\par
    \nobreak
    \endgroup
  \fi}
\makeatother

\usepackage{todonotes}

\newtheorem{theorem}{Theorem}[section]

\newtheorem{lemma}[theorem]{Lemma}
\newtheorem{corollary}[theorem]{Corollary}

\theoremstyle{definition}
\newtheorem{definition}[theorem]{Definition}
\newtheorem{remark}[theorem]{Remark}

\numberwithin{equation}{section}

\def \R {\mathbb {R}}

\def \E {\mathbb{E}}
\def \P {\mathbb{P}}
\def \Risk {\mathcal{R}}

\def\supp{\operatorname{supp}}
\def\osc{\operatorname{osc}}

\def\grad{\nabla}
\renewcommand{\div}{\operatorname{div}}
\renewcommand{\tilde}{\widetilde}

\allowdisplaybreaks

\begin{document}
	

\title[Neural ODEs]{Neural ODEs as the Deep Limit of\\ ResNets with constant weights}

\address{Benny Avelin \\Department of Mathematics, Uppsala University\\
	S-751 06 Uppsala, Sweden}
\email{benny.avelin@math.uu.se}

\address{Kaj Nystr\"{o}m\\Department of Mathematics, Uppsala University\\
S-751 06 Uppsala, Sweden}
\email{kaj.nystrom@math.uu.se}

\author{Benny Avelin and Kaj Nystr{\"o}m}

\maketitle

\date{\today}

\begin{abstract}
	\noindent In this paper we prove that, in the deep limit, the stochastic gradient descent on a ResNet type deep neural network, where each layer shares the same weight matrix, converges to the stochastic gradient descent for a Neural ODE and that the corresponding value/loss functions converge. Our result gives, in the context of minimization by stochastic gradient descent, a theoretical foundation for considering Neural ODEs as the deep limit of ResNets. Our proof is based on certain decay estimates for associated Fokker-Planck equations.\\

	\noindent
	2010   {Primary: 68T05, 65L20, Secondary: 34A45, 35Q84, 62F10, 60H10}
	\noindent

	\medskip


\end{abstract}
\maketitle


\makeatletter
\providecommand\@dotsep{5}

\setcounter{equation}{0} \setcounter{theorem}{0}

\section{Introduction}
\label{sec:intro}
\subsection*{ResNet and Neural ODEs}
Ever since the very popular ResNet paper \cite{He} was published several authors have made the observation that the ResNet architecture is structurally similar to the  Euler discretization of an ODE (\cite{EHL,LCW}). However, the original  `ResNet block' considered in \cite{He} is defined as
\begin{align}\label{resnet1}
	h_{t+1} = \sigma (h_t + K_t^{(1)} \sigma (K_t^{(2)}  h_t + b_t^{(2)} )+b_t^{(1)}), \quad t=0,\ldots, N-1
\end{align}
where $h_t,b_t^{(1)},b_t^{(2)} \in \R^d$ and $K_t^{(1)}, K_t^{(2)} : \R^d \to \R^d$ are linear operators, and $h_N$ represents the output of the network. The activation function $\sigma(x) = \max\{0,x\}$ is applied component-wise on vectors in $\R^d$. $N$ is the number of layers used in the construction. For standard neural networks the operators $K_t^{(1)},K_t^{(2)}$ are matrices, but for convolutional neural networks (CNNs) the operators are discrete convolution operators\footnote{Discrete convolutional operators can be phrased as matrices, see for instance \cite{Z}.}.

It is not immediately obvious that \cref{resnet1} can be viewed as the Euler discretization of an ODE as the activation function $\sigma$ is applied after the addition of $h_t$ and $K_t^{(1)} \sigma (K_t^{(2)} y_t + b_t^{(2)})+b_t^{(1)}$. However, removing\footnote{This was also done by the authors of ResNet in their improved follow-up paper \cite{He2}} the outermost activation $\sigma$ and instead consider the difference equation 
\begin{align}\label{resnet2}
	h_{t+1} = h_t + K_t^{(1)} \sigma (K_t^{(2)}  h_t + b_t^{(2)} )+b_t^{(1)}, \quad t=0,\ldots, N-1
\end{align}
we see that this is the Euler discretization (with time-step 1) of the ODE
\begin{align} \label{neuralODE---}
	\dot{h}_t = f(h_t,\theta_t), \quad t=[0,N],
\end{align}
where
\begin{align*} 
 f(\cdot,\theta):\mathbb R^d\to \mathbb R^d,\ f(x,\theta):= K^{(1)} \sigma (K^{(2)}  x + b^{(2)} )+b^{(1)},
\end{align*}
and where $\theta$ is the short notation for $\theta = (K^{(1)}, K^{(2)},b^{(1)},b^{(2)})$. Note that in \cref{neuralODE---} the time-step is 1 and hence the time horizon will be the number of layers $N$. This timescale is not optimal in the sense that if the ODE is stable then the system will attract to zero as $N \to \infty$. 

In this paper we consider (general) autonomous ODEs as in \cref{neuralODE---}, with time independent $\theta$ and a time horizon of $1$, i.e.\
\begin{align} \label{neuralODE--}
	\dot{h}_t = f(h_t,\theta), \quad t=[0,1].
\end{align}
This type of model is called a Neural ODE (NODE) \cite{CRBD} and has recently garnered a lot of attention as it has been proven to work very well in practice. Naturally we also consider the Euler discretization of \cref{neuralODE--}, which can be written as the difference equation 
\begin{align}\label{resnet3}
	h_{t_{i+1}} = h_{t_i} + \frac{1}{N} f(h_{t_i},\theta),\ i=0,...,N-1,\ t_i=i/N.
\end{align}
The models in \cref{neuralODE--,resnet3} map $\R^d$ into $\R^d$ for $t=1$. We call the models in \cref{neuralODE--,resnet3}, (general) ResNet type deep neural networks defined on the interval $[0,1]$. Furthermore, if the input data is in $\R^d$ and the output data in $\R^l$, we have to complement the model with a function mapping $\mathbb R^d\to \mathbb R^l$ resulting in the final output.

\subsection*{Empirical risk minimization}

Assume that we are given data distributed according to a probability measure $\mu$ where $(x,y) \sim \mu$, and $x \in \R^d$, $y \in \R^l$. Let $L : \R^l \times \R^l \to \R$ be a non-negative convex function. The learning problem for a model $h=h(\cdot,\theta):\R^d \times \R^m \to \R^l$, where $\theta\in \R^m$ indicates that the function $h$ is parameterized by weights $\theta$,  can be formulated as
\begin{align}\label{minima}
\min_{\theta}\mathcal{R}(\theta), \ \mathcal{R}(\theta):=\E_{(x,y) \sim \mu} \left [ L(h(x,\theta),y) \right ],
\end{align}
where $\mathcal{R}$ is often referred to as the \emph{risk} or the \emph{risk function}. In practice, the risk we have to minimize is the empirical risk, and it is a well-established fact that for neural networks the minimization problem in \cref{minima} is, in general, a non-convex minimization problem \cite{SafS,BGMSS, SJL, DuLee}. As such many search algorithms may get trapped at, or converge to, local minima which are not global minima \cite{SafS}. Currently, a variety of different methods are used in deep learning when training the model, i.e.\ when trying to find an approximate solution to the problem in \cref{minima}, we refer to \cite{VBGS} for an overview of various methods. One of the most popular methods, and perhaps the standard way of approaching the problem in \cref{minima}, is back-propagation using stochastic gradient descent, see \cite{HH} for a more recent outline of the method. While much emphasis is put on back-propagation in the deep learning community, from a theoretical perspective it does not matter if we use a forward or a backward mode of auto-differentiation.

\subsection*{Continuous approximation of SGD}

In  \cite{LTE,LTE2} it is proved that the stochastic gradient descent can be approximated by a continuous time process
\begin{align}\label{SGDfirst}
	d\theta_t = - \grad_\theta \mathcal{R} dt + \sqrt{\Sigma} dW_t,
\end{align}
where $\Sigma$ is a covariance matrix and $dW_t$ is a standard $m$-dimensional Wiener process defined on a probability space. The idea of approximating stochastic gradient descent with a continuous time process has been noted by several authors, see \cite{COOSC,CS,CLTZ,GL,MGIVW,MHB}.
A special case of what we prove in this paper, see \cref{thm:modified_risk} below, is that the stochastic gradient descent \cref{SGDfirst} used to minimize the risk for the ResNet model in \cref{resnet3} converges to the stochastic gradient descent used to minimize the risk for the Neural ODE model in \cref{neuralODE--}. This convergence is proved in the sense of expectation with respect to the random initialization of the weights in the stochastic gradient descent. Furthermore, we prove that the corresponding discrepancy errors decay as $N^{-1}$ where $N$ is the number of layers or discretization steps. 

\subsection*{Novelty and significance}

It is fair to say that in general there are very few papers making more fundamental and theoretical contributions to the understanding of deep learning and more specifically ResNet like neural networks. However, in the latter case there is a strain of recent and interesting contributions. In \cite{TVG} the authors  allow the parameters of the model to be layer- and time-dependent resulting in non-autonomous ODEs with corresponding Euler discretization:
\begin{align} \label{neuralODE}
	\dot{h}_t = f(h_t,\theta_t),\quad h_{t+1} = h_t + \frac{1}{N} f(h_t,\theta_t).
\end{align}
In particular, based on more restrictive assumptions on $f$, more restrictive compared to what we use,  it is proved in \cite{TVG} that as the number of layers tends to infinity in \cref{neuralODE}, the risk associated to \cref{neuralODE} and defined as in \cref{minima}, converges in the sense of gamma convergence to the risk associated to the corresponding (continuous) ODE in \cref{neuralODE}: we refer to Theorem 2.1 in  \cite{TVG} and to \cite{Maso} for an introduction to gamma convergence. The authors obtain that the minima for finite layer networks converge to minima of the continuous limit, infinite layer, network. To prove that the limit exists and has nice properties they introduce a regularization which penalizes the norm of the difference between the weights in subsequent layers. We emphasize that in \cite{TVG} the authors only consider the convergence of minimizers and not the convergence of the actual optimization procedure. In \cite{EHL} the authors study the limit problem directly and reformulates the problem as an optimal control problem for an ODE acting on measures. However, the complexity of such networks can be quite substantial due to the time-dependency of the weights, and it is unclear what would be the best way to construct a penalization such that the limit has nice properties.

As we mentioned before, in \cite{CRBD} the authors consider the autonomous ODE in \cref{neuralODE--}, i.e.\ they make the assumption that all layers share the same weights, and they develop two things. Firstly, they develop an adjoint equation that allows them to approximately compute the gradient with a depth independent memory cost. Secondly, they show through numerical examples that the approach works surprisingly well for some problems.

In general, the upshot of the ODE approach is the increased regularity, since trajectories are continuous and do not intersect, for autonomous ODEs, they are reversible, see \cite{DDW}. 
However, the increased regularity comes with a cost as Neural ODEs can have difficulties solving certain classification problems, see \cref{sec:numerical}.

Our main contribution is that we establish, in the context of minimization by stochastic gradient descent, a theoretical foundation for considering Neural ODEs as the deep limit ($N \to \infty$) of ResNets.

\subsection*{Overview of the paper}
The rest of the paper is organized as follows. In \cref{sec:main_result} we introduce the necessary formalism and notation and state our results: \cref{thm:penalization,thm:modified_risk,thm:modified_risk_msq}. In \cref{sec:resnet_error} we estimate, for $\theta$ fixed, the discretization error arising as a consequence of the Euler scheme, and we prove some technical estimates. In \cref{sec:SDE_FP} we collect and develop the results concerning stochastic differential equations and Fokker-Planck equations that are needed in the proofs of \cref{thm:penalization,thm:modified_risk,thm:modified_risk_msq}. \cref{sec:fp_estimates}  is devoted to the Fokker-Planck equations for the probability densities associated to the stochastic gradient descent for the Euler scheme and the continuous ODE, respectively. We establish some slightly delicate decay estimates for these densities, of Gaussian nature, assuming that the initialization density has compact support: see \cref{lemma1} below. In \cref{sec:proof_of_main_thm} we prove \cref{thm:penalization,thm:modified_risk,thm:modified_risk_msq}. In \cref{sec:numerical} we discuss a number of numerical experiments. These experiments indicate that in practice the rate of convergence is highly problem dependent, and that it can be considerably faster than indicated by our theoretical bounds. Finally, \cref{sec:conclusion} is devoted to a few concluding remarks.

\section{Statement of main results}
\label{sec:main_result}

Our main results concern (general) ResNet type deep neural networks defined on the interval $[0,1]$. To outline our setup we consider
\begin{align*}
	f_\theta:\R^d\to\R^d
\end{align*}
where $\theta\in \R^m$ is a potentially high-dimensional vector of parameters acting as a degree of freedom. Given $N\in\mathbb Z_+$ we consider  $[0,1]$ as divided into $N$ intervals each of length $N^{-1}$, and we define $x_i^{(N)}(x,\theta)$, $i = 0,\ldots, N$, recursively as
\begin{align}
	\label{diffeq}
	x^{(N)}_{i+1}(x,\theta)&= x^{(N)}_i(x,\theta) + \frac{1}{N} f_\theta(x^{(N)}_i(x,\theta)),\ i = 0,\ldots, N-1,\notag\\
x^{(N)}_0(x,\theta) &= x.
\end{align}
We define $x^{(N)}(t,x,\theta)=x^{(N)}_{i}(x,\theta)$ whenever $t\in [i/N,(i+1)/N)$. We will not indicate the dependency on $x$, $\theta$ when it is unambiguous.

We are particularly interested in the case when $f_\theta(x)$ is a general vector valued (deep) neural network having parameters $\theta$ but in the context of ResNets a specification for $f_\theta$ is, as discussed in the introduction,
\begin{align}\label{nnspec}
	f_\theta(x)=K^{(1)} \sigma (K^{(2)}  x + b^{(2)} )+b^{(1)}
\end{align}
where $\theta=(K^{(1)},K^{(2)},b^{(1)},b^{(2)})$ are parameters and $\sigma$ is a globally Lipschitz activation function, for instance sigmoid, tanh, ReLU, Leaky ReLU and Swish, \cite{RZL}. However, our arguments rely only on certain regularity and growth properties of $f_\theta$. We will formulate our results using the following classes of functions.
\begin{definition}
	\label{def:F} Let $g:\R_+ \to \R_+$, $g(0) \geq 1$, be a non-decreasing function. We say that the function $f_\theta(x):\R^m \times \R^d \to \R^d$ is in regularity class $\mathcal{A}(g)$ if
	\begin{equation}\label{eq3++a+l}
		\begin{aligned}
			\|f_\theta(x)-f_{\theta'}(x)\|&\leq  \max\{g(\|\theta\|),g(\|\theta'\|)\}\|\theta-\theta'\|\|x\|,\\
			\|f_\theta(x)-f_\theta(x')\|&\leq  g(\|\theta\|)\|x-x'\|,\\
			\|\grad_{\theta}f_\theta(x)-\grad_{\theta}f_\theta(x')\|&\leq  g(\|\theta\|)\max\{\|x\|,\|x'\|\}\|x-x'\|,\\
			\|\grad_{x}f_\theta(x)-\grad_{x}f_\theta(x')\|&\leq  g(\|\theta\|)\|x-x'\|,
		\end{aligned}
	\end{equation}
	whenever $x,x'\in\mathbb R^d$,  $\theta$, $\theta'\in\mathbb R^m$.
\end{definition}
Some remarks are in order concerning \cref{def:F}. Firstly, it contains the prototype neural network in \cref{nnspec}, with $g(s) = s+1$. Secondly, it is essentially closed under compositions, see \cref{lem:AgComposition} below. Therefore, finite layer neural networks satisfy \cref{def:F}. We defer the proof of the following lemma to \cref{sec:resnet_error}.

\begin{lemma}
	\label{lem:AgComposition}
	Let $f_\theta,g_\theta \in \mathcal{A}(g)$. Then $F_{\theta} = f_{\theta_2} \circ g_{\theta_1} \in \mathcal{A}(2g^3)$ for $\theta = (\theta_1,\theta_2)$.
\end{lemma}

Certain results in this paper require us to control the second derivatives of the risk. We therefore also introduce the following class of functions.
\begin{definition}
	\label{def:F+}
	Let $g:\R_+ \to \R_+$, $g(0) \geq 1$, be a non-decreasing function. We say that the function $f_\theta(x):\R^m \times \R^d \to \R^d$ is in regularity class $\mathcal{A}_+(g)$ if $f_\theta\in\mathcal{A}(g)$ and if there exists a polynomial $P:\R \to \R$ such that
	\begin{equation*}
		\begin{aligned}
			\|\grad_\theta^2 f_\theta(x)\| + \|\grad_\theta \grad_x f_\theta(x)\| &\leq g(\|\theta\|)P(\|x\|), \\
			\|\grad^2_x f\| &\leq g(\|\theta\|),
		\end{aligned}
	\end{equation*}
	whenever $x\in\mathbb R^d$,  $\theta \in\mathbb R^m$.
\end{definition}

The following lemma follows from \cref{lem:AgComposition} and an elementary calculation using \cref{def:F+}.

\begin{lemma}
	Let $f_\theta,g_\theta \in \mathcal{A}_+(g)$, then $F_{\theta} = f_{\theta_2} \circ g_{\theta_1} \in \mathcal{A}_+(2g^3)$ for $\theta = (\theta_1,\theta_2)$.
\end{lemma}

Given a probability measure $\mu$, $(x,y) \sim \mu$, $x \in \R^d$, $y \in \R^d$, and with $x^{(N)}(\cdot,\theta)$ defined as in \cref{diffeq}, we consider the penalized risk
\begin{align*}
	\mathcal{R}^{(N)}(\theta) := \E_{(x,y) \sim \mu} \left [ \|y-x^{(N)}(1,\theta)\|^2\right ]+\gamma H(\theta)
\end{align*}
where $\gamma\in \R_+$ is a  hyper-parameter and $H=H(\theta)$ is a non-negative and convex regularization. The finite layer model in \cref{diffeq} is, as described in \cref{sec:intro}, the forward Euler discretization of the autonomous system of ordinary differential equations
\begin{align}
	\label{ODE}
	\dot{x}(t)&=f_\theta(x(t)),\ t\in (0,1],\ x(0)=x,
\end{align}
where $x(t)=x(t,\theta) \in\mathbb R^d$. Given data from the distribution $\mu$ and with $x(\cdot,\theta)$ solving the system of Neural ODEs in \cref{ODE}, we consider the penalized risk
\begin{align*}
	\mathcal{R}(\theta) := \E_{(x,y) \sim \mu} \left [ \|y-x(1,\theta)\|^2\right ]+\gamma H(\theta).
\end{align*}
Throughout this paper we will assume that \textbf{moments of all orders are finite} for the probability measure $\mu$.
By construction the input data, $x$, as well as $x(t,\theta)$ are vectors in
$\mathbb R^d$. In the case when the output data is $\mathbb R^l$,  we need to modify the $\mathcal{R}^{(N)}(\theta)$ and $\mathcal{R}(\theta)$ by performing  final transformations of $x^{(N)}(1,\theta)$ and $x(1,\theta)$ achieving outputs in $\mathbb R^l$. These modifications are trivial to incorporate and throughout the paper we will in our derivations therefore simply assume that $l=d$.

We will base our analysis on the following  continuous in time approximations of the stochastic gradient descent, see \cite{LTE,LTE2},
\begin{align}\label{sgdsys}
	d\theta^{(N)}_t &= - \nabla \mathcal{R}^{(N)}(\theta^{(N)}_t)dt + \Sigma dW_t,\notag \\
	d\theta_t &= - \nabla \mathcal{R}(\theta_t)dt + \Sigma dW_t,
\end{align}
for $t \in [0,T]$.
Throughout the paper we will assume, for simplicity, that the constant covariance matrix $\Sigma$ has full rank something which, in reality, may not be the case, see \cite{CS}. We want to understand in what sense  $\theta^{(N)}_t$ approximates $\theta_t$ as $N \to \infty$. To answer this, we first need to make sure that $\theta^{(N)}_t$ and $\theta_t$ exist. Note  that $\nabla \mathcal{R}^{(N)}(\theta)$ and $\nabla \mathcal{R}(\theta)$ can, as functions of $\theta$, grow exponentially: simply consider the ODE $\dot{x} = \theta x$ which has $e^{\theta t}$ as a solution. This creates problems as the setting does not fit the standard theory of SDEs, see \cite{GS}, a theory which most commonly requires that the growth of the drift term is at most linear. However, if the drift terms are \emph{confining potentials}, i.e.\ $$-\nabla \mathcal{R}(\theta) \cdot \theta \leq c(1+\|\theta\|^2),\ -\nabla \mathcal{R}^{(N)}(\theta) \cdot \theta \leq c(1+\|\theta\|^2)$$ then we have existence and uniqueness for the SDEs in \cref{sgdsys}, see \cref{sec:SDE_FP}. In particular, if we have a bound on the growth of $\grad_\theta x^{(N)}(\cdot,\theta)$, $\grad_\theta x(\cdot,\theta)$ then, as we will see, we can choose $H$ to have sufficiently strong convexity to ensure the existence of a large constant $c$ such that the drift terms are confining potentials in the sense stated.

If we choose $H$ so that $\mathcal{R}^{(N)}$ and $\mathcal{R}$ are convex outside some large ball then $\mathcal{R}^{(N)}$ and $\mathcal{R}$ can be seen as bounded perturbations of strictly convex potentials, see the proof of \cref{thm:penalization}. Using this we can use the log-Sobolev inequality and hyper-contractivity properties of certain semi-groups, see \cref{sec:SDE_FP}, to obtain very good tail bounds for the densities of $\theta_t^{(N)}$ and  $\theta_t$. Actually these tail bounds are good enough for us to prove that the expected risks are bounded, expectation is over trajectories, and that $\theta^{(N)}_t \to \theta_t$ in probability. However, these bounds do not seem to be strong enough to give us quantitative convergence estimates for the difference $|\mathbb{E} [\mathcal{R}(\theta_T)]-\mathbb{E} [\mathcal{R}^{(N)}(\theta^{(N)}_T)]|$. The main reason for this is that even though we have  good tail bounds for the densities of $\theta^{(N)}_t$ and $\theta_t$ we do not have good estimates for $\theta^{(N)}_t - \theta_t$ in terms of $N$. The following is our first theorem.

\begin{theorem}
	\label{thm:penalization} Let $g:\R_+ \to \R_+$, $g(0) \geq 1$, be a non-decreasing function and assume that $f_\theta\in \mathcal{A}_+(g)$. Given $d$, $m$, there exists  a regularizing function $H=H(g)$ such that if we consider the corresponding penalized risks $\mathcal{R}^{(N)}$ and $\mathcal{R}$, defined using  $H$ and with $\gamma=1$, then  $\mathcal{R}^{(N)}$ and $\mathcal{R}$ are bounded perturbations of strictly convex functions. Furthermore, given $T>0$ and a compactly supported initial distribution $p_0$ for $\theta_0$, we have
	\begin{align*}
		\sup_{t\in[0,T]} \|\theta_t-\theta^{(N)}_t\| \to 0 \quad \text{in probability as $N \to \infty$}
	\end{align*}
	and
	\begin{align*}
		\mathbb{E} [\mathcal{R}(\theta_T)] < \infty, \quad \mathbb{E} [\mathcal{R}^{(N)}(\theta^{(N)}_T)] < \infty.
	\end{align*}
\end{theorem}

\begin{remark}
	\cref{thm:penalization} remains true but with a different rate of convergence,  if we replace $\gamma=1$ with $0<\gamma < 1$.
\end{remark}

There are a number of ways to introduce more restrictive assumptions on the risk in order to strengthen the convergence $\theta^{(N)}_t \to \theta_t$ and in order to obtain quantitative bounds for the difference $|\mathbb{E} [\mathcal{R}(\theta_T)]-\mathbb{E} [\mathcal{R}^{(N)}(\theta^{(N)}_T)]|$. Our approach is to truncate the loss function. This can be done in several ways, but a very natural choice is to simply restrict the hypothesis space by not allowing weights with too large norm. Specifically, we let $\Lambda>0$ be a large degree of freedom, and we consider
\begin{align}\label{eq:modified_risk-}
\tilde {\mathcal{R}}^{(N)}(\theta)&:= \E_{(x,y) \sim \mu} \left [ \|y-x^{(N)}(1,T_{\Lambda}(\theta))\|^2\right ]+\gamma H(\theta),\notag\\
\tilde {\mathcal{R}}(\theta)&:= \E_{(x,y) \sim \mu} \left [ \|y-x(1,T_{\Lambda}(\theta))\|^2\right ]+\gamma H(\theta),
\end{align}
instead of $\mathcal{R}^{(N)}$ and $\mathcal{R}$, where $T_{\Lambda}(\theta)$ is a smooth function such that  $T_{\Lambda}(\theta)= \theta$ if $\|\theta\| \leq {\Lambda}$ and $T_{\Lambda}(\theta) = \frac{\theta}{\|\theta\|} 2{\Lambda}$ if $\|\theta\| \geq 2 {\Lambda}$. It should be noted that, the choice of considering only the quadratic loss function is purely for simplicity and our technique works with minor modification for any other convex loss.

Having truncated the loss functions we  run continuous forms of SGDs 
\begin{align*}
	d\theta^{(N)}_t &= - \nabla \tilde {\mathcal{R}}^{(N)}(\theta^{(N)}_t)dt + \Sigma dW_t,  & \theta_0^{(N)} \sim p_0\\
	d\theta_t &= - \nabla \tilde {\mathcal{R}}(\theta_t)dt + \Sigma dW_t, & \theta_0 \sim p_0
\end{align*}
to minimize the modified risks. Using this setup, assuming also that $H(\theta)\approx \|\theta\|^2$ when $\|\theta\|$ is large, the modified risks in \cref{eq:modified_risk-} will satisfy quadratic growth estimates at infinity, and the modified risk will be globally Lipschitz. As a consequence all the tools from classical stochastic differential equations are at our disposal, see \cref{sec:SDE_FP}. This allows us to prove that $\theta^{(N)}_t\to \theta_t$ in the sense of mean square convergence. However, still the classical SDE theory does not easily seem to allow us to conclude  that $\tilde{\mathcal{R}}^{(N)}(\theta_t^{(N)})$ converges in any meaningful way to $\tilde{\mathcal{R}}(\theta_t)$. To overcome this difficulty we develop a PDE based approach to obtain further estimates for the densities of $\theta^{(N)}_t$ and $\theta_t$, and their differences. In particular, we prove the following theorem.
\begin{theorem}
	\label{thm:modified_risk} Let $g:\R_+ \to \R_+$, $g(0) \geq 1$, be a non-decreasing function and assume that $f_\theta\in \mathcal{A}(g)$. Let $\Lambda\gg 1$ and $T>0$ be fixed and assume that the initial weight randomization distribution $p_0$ has compact support in $B(0,R_0)$, $R_0\geq 1$. Assume also that $$\lambda^{-1}\|\theta\|^2\leq H(\theta)\leq \lambda\|\theta\|^2$$ on
$\mathbb R^m\setminus B(0,\rho_0)$, $\rho_0\geq 1$ and for some $\lambda\geq 1$. Then for given $\gamma > 0$, $\Lambda > 0$ and positive definite $\Sigma$, there exists a positive and finite constant $c$, depending on the function $g$  as well as  $d$, $m$, $\Lambda$, $\gamma$, $\Sigma$, $T$, $R_0$,  $\lambda$ and $\rho_0$, such that
	\begin{align}
		\label{thm2.4a}
		\sup_{t\in[0,T]}\bigl \|\E[\theta_t-\theta^{(N)}_t]\bigr \|&\leq   c N^{-1}\|p_0\|_2,\\
		\label{thm2.4b}
		\sup_{t\in[0,T]}\bigl |\E[\tilde {\mathcal{R}}(\theta_t)-\tilde {\mathcal{R}}^{(N)}(\theta^{(N)}_t) ]\bigr |&\leq  c N^{-1}\|p_0\|_2.
	\end{align}
Furthermore, if $\tilde R_0\geq 2R_0$, then
	\begin{align*}
		\sup_{t\in[0,T]}\bigl \|
		\E[\theta_t | \theta_t \in B_{\tilde R_0}]-\E[\theta^{(N)}_t | \theta^{(N)}_t \in B_{\tilde R_0}] \bigr \|&\leq   c N^{-1}e^{-\tilde R_0^2/T}\|p_0\|_2,\notag\\
		\sup_{t\in[0,T]}\bigl |
		\E[\tilde \Risk(\theta_t) | \theta_t \in B_{\tilde R_0}]-\E[\tilde \Risk^{(N)}(\theta^{(N)}_t) | \theta^{(N)}_t \in B_{\tilde R_0}] \bigr |&\leq   c N^{-1}e^{-\tilde R_0^2/T}\|p_0\|_2.\notag
	\end{align*}
\end{theorem}

\begin{remark}
	The requirement that the regularization term $H$ grows quadratically is crucial for our quantitative estimates, and as such our result only covers $L^2$ regularization (weight decay) and not sparse type regularizations like $L^1$.
\end{remark}

To prove \cref{thm:modified_risk}  we develop certain estimates for $p^{(N)}(\theta,t)$, $p(\theta,t)$, i.e.\ for the to $\theta^{(N)}(t)$, $\theta(t)$,  associate probability densities, by exploring
the associated Fokker-Planck equations: see \cref{eq2a}. In fact, we prove several estimates which give that $p^{(N)}(\theta,T)\to p(\theta,T)$ in a very strong sense, stronger than initially is recognized from the statement of \cref{thm:modified_risk}. Particular consequences of our proofs are the estimates
\begin{align}\label{esta0}
	\int\limits_{\mathbb R^m}e^{\gamma H(\theta)/4}(p(\theta,T)+p^{(N)}(\theta,T))\, d\theta&\leq  c \|p_0\|_2,\\
\int\limits_{B(0, 2^{k+1}\tilde R_0)\setminus B(0, 2^{k}\tilde R_0)}e^{\gamma H(\theta)/4}(p(\theta,T)+p^{(N)}(\theta,T))\, d\theta&\leq  c e^{-2^k\tilde R_0^2/T} \|p_0\|_2, \notag
\end{align}
and
\begin{align} \label{esta1}
	\int\limits_{\mathbb R^m}e^{\gamma H(\theta)/4}|p(\theta,T)-p^{(N)}(\theta,T)|\, d\theta&\leq  c N^{-1}\|p_0\|_2,\\
\int\limits_{B(0, 2^{k+1}\tilde R_0)\setminus B(0, 2^{k}\tilde R_0)}e^{\gamma H(\theta)/4}|p(\theta,T)-p^{(N)}(\theta,T)|\, d\theta&\leq  c e^{-2^k\tilde R_0^2/T} N^{-1}\|p_0\|_2, \notag
\end{align}
whenever $\tilde R_0\geq 2R_0$, $k\in\mathbb N$ and where $c= c(g,d,m, \Lambda, \gamma, \Sigma, T,R_0, \lambda, \rho_0)$. In particular, these estimates indicate that
$p(\theta,T)$, $p^{(N)}(\theta,T)$ and $|p(\theta,T)-p^{(N)}(\theta,T)|$ have Gaussian tails away from the (compact) support of $p_0$.
\begin{remark}
	The estimates \cref{esta0,esta1} can be interpreted from a probabilistic point of view. The bound \cref{esta0} for $p$ is equivalent to $\E[e^{\gamma H(\theta_T)/4}] \leq c \|p_0\|_2$, which implies that all moments of $H(\theta_T)$ are finite. Secondly we can interpret \cref{esta1} as the total variation distance between $p$ and $p^{(N)}$ being of order $N^{-1}$ with respect to the measure $e^{\gamma H(\theta)/4} d\theta$.
\end{remark}
A direct consequence of \cref{esta1} is the following corollary which states that $\theta_t^{(N)}$ is a weak order 1 approximation of $\theta_t$.
\begin{corollary}
	Let $\theta_t,\theta_t^{(N)}$ be as in \cref{thm:modified_risk}. Let $\varphi: \R^m \to \R$ be a continuous function satisfying the growth condition
	\begin{align*}
		|\varphi(x)| \leq P(x), \quad x \in \R^m,
	\end{align*}
	for some polynomial $P:\R^m \to \R$ of order $k$.
	Then there exists a constant $c$ depending on $g$  as well as  $d$, $m$, $\Lambda$, $\gamma$, $\Sigma$, $T$, $R_0$,  $\lambda, \rho_0$ and $P$ such that
	\begin{align*}
		\sup_{t\in[0,T]}\bigl |\E[\varphi(\theta_t)-\varphi(\theta^{(N)}_t)]\bigr |&\leq   c N^{-1}\|p_0\|_2.
	\end{align*}
\end{corollary}
 Using our estimates we can also derive the following theorem from the standard theory.
\begin{theorem}
	\label{thm:modified_risk_msq}
	Let $g:\R_+ \to \R_+$, $g(0) \geq 1$, be a non-decreasing function and assume that $f_\theta\in \mathcal{A}(g)$. Let $\Lambda\gg 1$ and $T>0$ be fixed and assume that $\E[\|\theta_0\|^2]< \infty$. Then
		\begin{align*}
			\E [\sup_{t\in[0,T]} \|\theta(t)-\theta^{(N)}(t)\|^2  ] \to 0.
		\end{align*}
\end{theorem}

\section{ResNet and the Neural ODE: error analysis}\label{sec:resnet_error}
In this section we will estimate the error formed by the discretization of the ODE. This will be done on the level of the trajectory as well as on the level of the first and second derivatives with respect to $\theta$. Note that by construction, see \cref{diffeq,ODE} we have
\begin{align*}
	x^{(N)}(t,x,\theta)&= x+\int_0^tf_\theta(x^{(N)}(\tau,x,\theta))\, d\tau,\notag\\
	x(t,x,\theta)&= x+\int_0^t f_\theta(x(\tau,x,\theta))\, d\tau,
\end{align*}
for all $t\in [0,1]$. We assume consistently that the paths $x^{(N)}(t,x,\theta)$ and $x(t,x,\theta)$ start at $x$ for $t=0$, and are driven by the parameters $\theta$. We will in the following, for simplicity, use the notation $x^{(N)}(t):=x^{(N)}(t,x,\theta)$, $x_i^{(N)}(t)=x^{(N)}(i/N)$, $x(t):=x(t,x,\theta)$. Recall that $\tilde{\mathcal{R}}^{(N)}$ and $\tilde{\mathcal{R}}$ are introduced in \cref{eq:modified_risk-} using the cut-off parameter $\Lambda$. Let
$\hat{\mathcal{R}}^{(N)}(\theta):=\tilde{\mathcal{R}}^{(N)}(\theta)-\gamma H(\theta)$ and $\hat{\mathcal{R}}(\theta):=\tilde{\mathcal{R}}(\theta)-\gamma H(\theta)$. In this section we prove estimates on the discrepancy error between the trajectories of $x(t)$ and the discrete trajectories $x_i^{(N)}$. We begin by bounding the difference.

\begin{lemma}
	\label{lemmadiff1}Let $g:\R_+ \to \R_+$, $g(0) \geq 1$, be a non-decreasing function and assume that $f_\theta\in \mathcal{A}(g)$ as in \cref{def:F}. Then there exists a function $\tilde g:\mathbb R_+\to\mathbb R_+$ such that 
	\begin{equation}
		\label{eq:xdiff}
		\left \|x \left ({i}/{N}\right ) - x_i^{(N)} \right \| \leq \frac {\tilde g(g(\|\theta\|))}N \|x\|
	\end{equation}
	holds for all $i \in \{0,\ldots,N\}$, and such that
	\begin{equation}
		\label{eq:xbdd}
		\big \|x \left ({i}/{N}\right ) \big \| + \big \|x_i^{(N)} \big \| \leq  {\tilde g(g(\|\theta\|))} \|x\|,
	\end{equation}
	holds for all $i \in \{0,\ldots,N\}$.
\end{lemma}
\begin{proof}
	To start the proof we first write
	\begin{multline}
		\label{eq:discContDiff}
		\left \|x({i}/{N}) - x_i^{(N)} \right \|= \left \| x (({i-1})/{N}) + \int_{\frac{i-1}{N}}^{\frac{i}{N}} \dot{x}(t) dt - x_{i-1}^{(N)} - (x_{i}^{(N)}-x_{i-1}^{(N)}) \right \| \\
		\leq \left \| x( ({i-1})/{N}) - x_{i-1}^{(N)}\right \| + \left \| \int_{\frac{i-1}{N}}^{\frac{i}{N}} \dot{x}(t) dt - (x_i^{(N)} - x_{i-1}^{(N)}) \right \|.
	\end{multline}
	To bound the second term in \cref{eq:discContDiff} we note that
	\begin{multline}
		\label{eq:discContDiff2}
		\left \| \int_{\frac{i-1}{N}}^{\frac{i}{N}} \dot{x}(t) dt - (x_i^{(N)} - x_{i-1}^{(N)}) \right \| = \left \| \int_{\frac{i-1}{N}}^{\frac{i}{N}} f_\theta(x(t))  -
		f_\theta( x_{i-1}^{(N)}) dt \right \| \\
		\leq g(\|\theta\|) \int_{\frac{i-1}{N}}^{\frac{i}{N}}  \left \|x(t) - x_{i-1}^{(N)} \right \|dt,
	\end{multline}
	by \cref{def:F}. By the triangle inequality, the integral on the right-hand side of \cref{eq:discContDiff2} is bounded by
	\begin{multline}
		\label{eq:discContDiff2int}
		\int_{\frac{i-1}{N}}^{\frac{i}{N}} \left \| x(t) - x_{i-1}^{(N)} \right \| dt\leq  \frac{1}{N} \left \| x_{i-1}^{(N)} - x(({i-1})/{N}) \right \| \\
	+ \int_{\frac{i-1}{N}}^{\frac{i}{N}}  \left \| x(t) - x(({i-1})/{N}) \right \|dt.
	\end{multline}
	By the definition of $x(t)$ and \cref{def:F} we have that the second term in \cref{eq:discContDiff2int} is bounded as
	\begin{multline} \label{eq:second}
	\int_{\frac{i-1}{N}}^{\frac{i}{N}}\left \|x(t) - x(({i-1})/{N}) \right \| dt \leq \int_{\frac{i-1}{N}}^{\frac{i}{N}} \int_{\frac{i-1}{N}}^{t}\|f_\theta(x(r))\|\, dr dt \\
	\leq g(\|\theta\|)\int_{\frac{i-1}{N}}^{\frac{i}{N}} \int_{\frac{i-1}{N}}^{t}\|x(r)\|\, dr dt+\frac{1}{N^2}\|x\|
	\leq \frac {2g(\|\theta\|)}{N^2}\|x\|_{L^{\infty}([0,1])}
	\end{multline}
	whenever $t\in [(i-1)/N,i/N]$. Assembling \cref{eq:discContDiff,eq:discContDiff2,eq:discContDiff2int,eq:second} we arrive at
	\begin{align}
		\label{eq:Aibound}
		A_i \leq \left (1+\frac{g(\|\theta\|)}{N} \right ) A_{i-1} + 2\frac{g^2(\|\theta\|)}{N^2} \|x\|_{L^\infty([0,1])}
	\end{align}
	where $A_i = \left \|x({i}/{N}) - x_i^{(N)} \right \|$ for $i = 0,1,\ldots,N$, and $A_0 = 0$. \cref{eq:Aibound} can be rewritten as $A_i \leq C_0 A_{i-1} + C_1$. Iterating this inequality we see that
	\begin{multline}
		\label{eq:AiIteration}
		A_k \leq C_0 A_{k-1} + C_1 \leq C_0 (C_0 A_{k-1} + C_1) + C_1 = C_0^2 A_{k-2} + C_0C_1 + C_1 \\
		\leq C_0^{k} A_0 + C_1 \sum_{j=0}^{k-1} C_0^j = C_1 \sum_{j=0}^{k-1} C_0^j \leq C_1 \frac{C_0^N - 1}{C_0-1},
	\end{multline}
	for any  $k \in \{1,\ldots, N\}$. Elementary calculations give us
	\begin{align}
		\label{eq:diffxbound}
		\left \|x({i}/{N}) - x_i^{(N)} \right \| \leq & \frac {2g^2(\|\theta\|)}{N}e^{g(\|\theta\|)}\|x\|_{L^\infty([0,1])}.
	\end{align}
	To finalize the proof of \cref{eq:xdiff} we need to establish a bound for $\|x\|_{L^\infty([0,1])}$ in terms of the initial value $\|x\|$.
	By the definition of $x(t)$ and \cref{def:F} we have for any $t_1 \in [0,1]$ such that $t_1+t \leq 1$
	\begin{align*}
		\|x(\cdot)\|_{L^{\infty}([t_1,t_1+t])} &= \sup_{s \in [t_1,t_1+t]} \left \|\int_{t_1}^{t_1+s} \dot{x}(r) dr + x({t_1}) \right \| \\
	&= \sup_{s \in [t_1,t_1+t]} \left \|\int_{t_1}^{t_1+s} f_\theta(x(r)) dr + x({t_1}) \right \|\\
		&\leq g(\|\theta\|)t \|x(\cdot)\|_{L^{\infty}([t_1,t_1+t])} +t\|f_\theta(x(t_1))\|+\|x({t_1})\|\\
	&\leq g(\|\theta\|)t \|x(\cdot)\|_{L^{\infty}([t_1,t_1+t])} + (g(\|\theta\|)t+1)\|x({t_1})\|+\|x\|.
	\end{align*}
	Choosing $t$ small enough i.e.\ $t = \frac{1}{2 g(\|\theta\|)}$ we deduce
	\begin{align*}
		\|x\|_{L^{\infty}([t_1,t_1+t])} &\leq 4(\|x_{t_1}\|+\|x\|).
	\end{align*}
	Iterating the above inequality $N =2  \lceil g(\|\theta\|) \rceil$ times we obtain
	\begin{align}\label{bound2}
		\|x\|_{L^{\infty}([0,1])} &\leq  4^{2  g(\|\theta\|)}\|x\|.
	\end{align}
	Combining \cref{eq:diffxbound,bound2} proves \cref{eq:xdiff}. 

	It remains to prove \cref{eq:xbdd} and to start the proof note that the first term on the left in \cref{eq:xbdd} is already bounded by \cref{bound2}. Thus, we only need to establish that $\|x_i^{(N)}\|$ is bounded. We first note, using the definition of $x_i^{(N)}$ and \cref{def:F}, that
	\begin{align*}
		\left \| x_{i+1}^{(N)} - x_{i}^{(N)}\right \|\leq \frac{2g(\|\theta\|)}{N} \left ( \|x_i^{(N)} \|+\|x \|\right ).
	\end{align*}
	By the triangle inequality we get by rearrangement that
	\begin{align*}
		\| x_{i+1}^{(N)} \| \leq \left ( 1 + \frac{2g(\|\theta\|)}{N} \right ) \|x_{i}^{(N)}\| + \frac{2g(\|\theta\|)}{N}\|x \|.
	\end{align*}
	This is again an estimate of the form $A_i \leq C_0 A_{i-1} + C_1$, where $A_0 = \|x\|$. Iterating this we obtain as in \cref{eq:AiIteration} that
	\begin{align*}
		A_i \leq C_0^N A_0 + C_1 \frac{C_0^N - 1}{C_0-1},
	\end{align*}
	and by elementary calculations as in \cref{eq:diffxbound} we can conclude that
	\begin{align*}
		\|x_{i}^{(N)}\| \leq e^{2 g(\|\theta\|)} \|x\|, \quad i=0,1,\ldots,N.
	\end{align*}
	This proves the final estimate stated in \cref{eq:xbdd} and finishes the proof.
\end{proof}

We next upgrade the previous lemma to the level of the gradient of the trajectories with respect to $\theta$.

\begin{lemma} \label{lemmadiff20}
	Let $g:\R_+ \to \R_+$, $g(0) \geq 1$, be a non-decreasing function and assume that $f_\theta\in \mathcal{A}(g)$ as in \cref{def:F}. Then there exists a function $\tilde g:\mathbb R_+\to\mathbb R_+$ such that 
		\begin{equation}
			\label{eq:xdiffgrad}
			\left \|\grad_\theta x \left ({i}/{N}\right ) - \grad_\theta x_i^{(N)} \right \| \leq \frac {\tilde g(g(\|\theta\|))}N \|x\|
		\end{equation}
		holds for all $i \in \{0,\ldots,N\}$, and such that
		\begin{equation}
			\label{eq:xbddgrad}
			\big \|\grad_\theta x \left ({i}/{N}\right ) \big \| + \big \|\grad_\theta x_i^{(N)} \big \| \leq  {\tilde g(g(\|\theta\|))} \|x\|,
		\end{equation}
	holds for all $i \in \{0,\ldots,N\}$.
\end{lemma}
\begin{proof}
	We begin by proving \cref{eq:xdiffgrad} under the assumption that \cref{eq:xbddgrad} holds.
	First note that
	\begin{multline}
		\label{eq:5.1}
		\| \grad_\theta x(i/N)-\grad_\theta x^{(N)}_i\|
		\leq \|\grad_\theta x_{i-1}^{(N)} - \grad_\theta x((i-1)/N)\| \\
		+\|\grad_\theta x_i^{(N)} - \grad_\theta x_{i-1}^{(N)} - \grad_\theta \int_{\frac{i-1}{N}}^{\frac{i}{N}} \dot{x}(r)\, dr \|.
	\end{multline}
	As in the proof of \cref{lemmadiff1} we want to build iterative inequalities, and our first goal is to bound the second term in \cref{eq:5.1}. First note that
	\begin{multline*}
		\bigg \|\grad_\theta x_i^{(N)} - \grad_\theta x_{i-1}^{(N)} - \grad_\theta \int_{\frac{i-1}{N}}^{\frac{i}{N}} \dot{x}(r)\, dr \bigg \| \notag\\
		=  \left \|\frac{1}{N}\grad_\theta f_\theta(x_{i-1}^{(N)}) - \grad_\theta \int_{\frac{i-1}{N}}^{\frac{i}{N}} f_\theta(x(r))\, dr \right \|\notag \\
		=  \left \|\int_{\frac{i-1}{N}}^{\frac{i}{N}} \grad_\theta(f_\theta(x(r) - f_\theta(x_{i-1}^{(N)}) )\, dr \right \|.
	\end{multline*}
	Now using \cref{def:F} multiple times we see that
	\begin{align}
		\label{eq:5.1B}
		\|\grad_\theta (f_\theta(x(r)) - &f_\theta(x_{i-1}^{(N)}) )\|\notag
			\\
			\leq& \|\grad_x f_\theta(x(r))\grad_\theta x(r)-\grad_x f_\theta(x_{i-1}^{(N)})\grad_\theta x_{i-1}^{(N)}\|\notag
			\\
			&+g(\|\theta\|)\max\{\|x(r)\|,\|x_{i-1}^{(N)}\|\}\|x(r)-x_{i-1}^{(N)}\|\notag\\
			\leq&\|\grad_x f_\theta(x(r))\|\|\grad_\theta x(r)-\grad_\theta x_{i-1}^{(N)}\|\notag\\
			&+\|\grad_x f_\theta(x(r))-\grad_x f_\theta(x_{i-1}^{(N)})\|\|\grad_\theta x_{i-1}^{(N)}\|\\
			&+g(\|\theta\|)\max\{\|x(r)\|,\|x_{i-1}^{(N)}\|\}\|x(r)-x_{i-1}^{(N)}\|, \notag\\
			\leq & g(\|\theta\|)\|\grad_\theta x(r)-\grad_\theta x_{i-1}^{(N)}\|\notag\\
						&+g(\|\theta\|)\|x(r)-x_{i-1}^{(N)}\|(\max\{\|x(r)\|,\|x_{i-1}^{(N)}\|\}+\|\grad_\theta x_{i-1}^{(N)}\|). \notag
	\end{align}
	We want to bound the terms on the right-hand side in \cref{eq:5.1B} and to do this we first note that
	\begin{multline}
		\label{eq:5.gradABdiff21}
		\|\grad_\theta x(r) - \grad_\theta x_{i-1}^{(N)}\| \leq \|\grad_\theta x(r) - \grad_\theta x\bigl ({({i-1})/{N}}\bigr )\|\\
	 + \|\grad_\theta x\bigl ({({i-1})/{N}}\bigr ) - \grad_\theta x_{i-1}^{(N)}\|.
	\end{multline}
	The second term appearing in \cref{eq:5.gradABdiff21} is what we want to bound in an iterative scheme. Focusing on the first term in \cref{eq:5.gradABdiff21}, using \cref{def:F}, we see that
	\begin{align}
		\label{eq:5.1C}
		\notag \|\grad_\theta x(r) - \grad_\theta x\bigl ({({i-1})/{N}}\bigr )\|&=\bigg \|\grad_\theta \int_{\frac{i-1}{N}}^r \dot{x}(s)\, ds \bigg \|\notag\\
		&=\bigg \|\grad_\theta \int_{\frac{i-1}{N}}^r f_\theta (x(s))\, ds \bigg \|\notag\\
		&\leq \bigg \|\int_{\frac{i-1}{N}}^r \grad_x f_\theta (x(s))\grad_\theta x(s)+\grad_\theta f_\theta (x(s))\, ds \bigg \|\notag\\
		&\leq g(\|\theta\|) \int_{\frac{i-1}{N}}^r  ( \| x(s)\| + \|\grad_\theta x(s)\|  )\, ds.
	\end{align}
	Again, since we assume \cref{eq:xbddgrad} we can use \cref{eq:5.1,eq:5.1B,eq:5.1C,eq:5.gradABdiff21,lemmadiff1} to get
	\begin{align*}
		A_i \leq  \left ( 1 + \frac{g}{N} \right )A_{i-1} + \frac{1}{N^2} g_1(g(\| \theta \|))\max\{\|x\|,\|x\|^2\},
	\end{align*}
	for some non-decreasing function $g_1$, $A_i := \| \grad_\theta x(i/N)-\grad_\theta x^{(N)}_i\|$, and $A_0 := 0$. The above is an iterative inequality of the same type as \cref{eq:Aibound}, and we get
	\begin{align*}
		\| \grad_\theta x(i/N)-\grad_\theta x^{(N)}_i\| \leq \frac{g_2(g(\|\theta\|))}{N} \max\{\|x\|,\|x\|^2\}, \quad i=0,1,\ldots,N,
	\end{align*}
	where $g_2$ is another non-decreasing function. This completes the proof of \cref{eq:xdiffgrad} under the assumption \cref{eq:xbddgrad}.
	
	We now prove \cref{eq:xbddgrad}. Let $t \in [0,1]$ and $\epsilon > 0$, and note that from \cref{def:F} we have
	\begin{multline*}
		\sup_{r \in  [t,t+\epsilon]} \|\grad_\theta x(r)\| = \sup_{r \in  [t,t+\epsilon]} \bigg \| \grad_\theta \bigg ( \int_{t}^{r} f_\theta(x(s)) ds + x(t)\bigg ) \bigg\| \\
		\leq g(\|\theta\|) \epsilon  \sup_{r \in  [t,t+\epsilon]}\bigg ( \| x(r)\|+ \|\grad_\theta x(r)\| \bigg ) + \|\grad_\theta x(t)\|.
	\end{multline*}
	Fix $\epsilon=1/(2g(\|\theta\|))$. Then
	\begin{align*}
		\sup_{r \in [t,t+\epsilon]}\|\grad_\theta x(r)\| \leq \frac 1 2\| x(r)\| + \frac {1}{2g(\|\theta\|)} \|\grad_\theta x(t)\|.
	\end{align*}
	Iterating this inequality, using also \cref{eq:xbdd}, we get
	\begin{align*} 
		\sup_{r \in [0,1]} \|\grad_\theta x(r)\| \leq \tilde g(\|\theta\|)\|x\|,
	\end{align*}
	for some non-decreasing function $\tilde g:\mathbb R_+\to\mathbb R_+$. Similar bounds can be deduced for $\grad_\theta x^{(N)}$. This establishes \cref{eq:xbddgrad} and completes the proof of the lemma.
\end{proof}

In the proof of \cref{thm:penalization} we need to control the second derivatives of the risk and thus we need to bound the second derivatives of the trajectories w.r.t. $\theta$. We have collected what is needed in the following lemma.

\begin{lemma} \label{lemmagrowth}
	Let $g:\R_+ \to \R_+$, $g(0) \geq 1$, be a non-decreasing function and assume that $f_\theta\in \mathcal{A}_+(g)$. Then there exists a non-decreasing positive function $\hat g:\R_+ \to \R_+$ such that
	\begin{align*}
		\sup_{t \in [0,1]} \|\grad^2_\theta x(t)\| \leq \hat g(\|\theta\|)\tilde P(\|x\|),
	\end{align*}
	for some polynomial $\tilde P$.
\end{lemma}

\begin{proof}
	We will proceed similarly as in the proof of \cref{lemmadiff20}.
	Starting with $t \in [0,1]$ and $\epsilon > 0$, using \cref{def:F+} we have
	\begin{multline*}
		\sup_{r \in  [t,t+\epsilon]} \|\grad^2_\theta x(r)\| = \sup_{r \in  [t,t+\epsilon]} \bigg \| \grad^2_\theta \bigg ( \int_{t}^{r} f_\theta(x(s)) ds + x(t)\bigg ) \bigg\| \\
		\leq g^3(\|\theta\|) \epsilon \left ( \sup_{r \in  [t,t+\epsilon]} \|\grad^2_\theta x(r)\| + \tilde P(\|x\|) \right )+ \|\grad^2_\theta x(t)\|,
	\end{multline*}
	for a polynomial $\tilde P$.
	This gives, with $\epsilon=1/(2g^3(\|\theta\|))$ and after absorption, using \cref{lemmadiff1,lemmadiff20}, that
	\begin{align*}
		\sup_{r \in  [t,t+\epsilon]} \|\grad^2_\theta x(r)\| \leq \tilde P(\|x\|) + \frac{1}{2g^3(\theta)} \|\grad^2_\theta x(t)\|.
	\end{align*}
	Iterating this inequality we deduce that
	\begin{align*}
		\sup_{r \in [0,1]} \|\grad^2_\theta x(r)\| \leq \hat g(\|\theta\|)\tilde P(\|x\|)
	\end{align*}
	for some function $\hat g:\mathbb R_+\to\mathbb R_+$. Similar bounds can be deduced for $\grad^2_\theta x^{(N)}$.
\end{proof}

The next lemma of this section uses \cref{lemmadiff1,lemmadiff20} to bound the difference of the risks for the discrete and the continuous systems, as well as the difference at the gradient level. Recall that $\hat \Risk, \hat \Risk^{(N)}$ are the un-penalized truncated risks as defined at the beginning of this section.

\begin{lemma}
	\label{lemmadiff2} Let $g:\R_+ \to \R_+$, $g(0) \geq 1$, be a non-decreasing function and assume that $f_\theta\in \mathcal{A}(g)$. 
	Then
	\begin{align*}
		|\hat{\mathcal {R}}^{(N)}(\theta)|+|\hat{\mathcal {R}}(\theta)|+\|\nabla \hat{\mathcal {R}}^{(N)}(\theta)\|+\|\nabla \hat{\mathcal {R}}(\theta)\|&\leq  c(g,\Lambda),\notag\\
\|\nabla^2 \hat{\mathcal {R}}^{(N)}(\theta)\|+\|\nabla^2 \hat{\mathcal {R}}(\theta)\|&\leq  c(g,\Lambda),\notag\\
		|\hat{\mathcal {R}}(\theta)-\hat{\mathcal {R}}^{(N)}(\theta)|+\|\nabla (\hat{\mathcal {R}}(\theta)-\hat{\mathcal {R}}^{(N)}(\theta))\|&\leq  \frac 1 N c(g,\Lambda).
	\end{align*}
\end{lemma}
\begin{proof}
	We will only prove the estimates
	\begin{align}\label{term1}
		\|\nabla \hat{\mathcal {R}}^{(N)}(\theta)\|+\|\nabla \hat{\mathcal {R}}(\theta)\|&\leq  c(g,\Lambda),\notag\\
		\|\nabla (\hat{\mathcal {R}}(\theta)-\hat{\mathcal {R}}^{(N)}(\theta))\|&\leq  \frac 1 N c(g,\Lambda),
	\end{align}
	as all the other estimates are proved similarly. To get started we first note that
	\begin{align}\label{term2}
		\grad \hat{\mathcal {R}}^{(N)}(\theta)&= 2\E_{(x,y) \sim \mu} \left [ (y-x^{(N)}_1) \grad_\theta x^{(N)}_N\right ],\notag\\
		\grad \hat{\mathcal {R}}(\theta)&= 2\E_{(x,y) \sim \mu} \left [ (y-x(1)) \grad_\theta x(1)\right ],
	\end{align}
	where now $x^{(N)}_N=x^{(N)}_N(T_\Lambda(\theta))$, $x(1)=x(1,T_\Lambda(\theta))$. As $T_\Lambda(\theta)$ is constant when $\|\theta\|\geq 2\Lambda$, and
	$\|\nabla_\theta T_\Lambda(\theta)\|$ is bounded, we see that it suffices to simply derive the  bounds when $\|\theta\|\leq \Lambda$. In this case $T_\Lambda(\theta)=\theta$. Using \cref{term2} we see that to estimate the terms in \cref{term1} it suffices to bound $\|x(1)\|,\|x_N^{(N)}\|,\|\grad_\theta x^{(N)}_N\|$, $\|\grad_\theta x(1)\|$, $\| x(1)- x^{(N)}_N\|$, and $\| \grad_\theta x(1)-\grad_\theta x^{(N)}_N\|$ which are all provided by \cref{lemmadiff1,lemmadiff20}.
\end{proof}

We end this section with the proof of \cref{lem:AgComposition}.
\subsection*{Proof of \cref{lem:AgComposition}}
The proof is just a matter of verifying each part of \cref{eq3++a+l}. We begin with the Lipschitz continuity in $\theta$. From the triangle inequality, and repeatedly applying \cref{def:F}, we get
\begin{align*}
	\|F_\theta(x)-F_{\theta'}(x)\| 
	\leq& g(\|\theta_2\|) \max\{g(\|\theta_1\|),g(\|\theta'_1\|)\}\|\theta_1-\theta'_1\|\|x\| \\
	&+\max\{g(\|\theta_2\|),g(\|\theta'_2\|)\}\|\theta_2-\theta'_2\| g(\|\theta'_1\|)\|x\| \\
	\leq& 2\max\{g^2(\|\theta\|),g^2(\|\theta'\|)\}\|\theta-\theta'\| \|x\|.
\end{align*}
The Lipschitz continuity in $x$ is a simple consequence of \cref{def:F}
\begin{align*}
	\|F_\theta(x)-F_\theta(x')\| 
	\leq g(\|\theta_2\|)g(\|\theta_1\|)\|x-x'\| 
	\leq g^2(\|\theta\|)\|x-x'\|.
\end{align*}
For the $\theta$ gradient of $F$ the Lipschitz continuity requires a bit more work, and we note that
\begin{align}
	\label{eq:1}
	\|\grad_{\theta}F_\theta(x)-\grad_{\theta}F_\theta(x')\|^2 =& \|\grad_{\theta_1} (f_{\theta_2} \circ g_{\theta_1}(x))-\grad_{\theta_1} (f_{\theta_2} \circ g_{\theta_1}(x'))\|^2 \notag \\
	&+ \|\grad_{\theta_2} (f_{\theta_2} \circ g_{\theta_1}(x))-\grad_{\theta_2} (f_{\theta_2} \circ g_{\theta_1}(x'))\|^2.
\end{align}
From \cref{def:F} we see that the first term in \cref{eq:1} is bounded as
\begin{align*}
	\|\grad_{\theta_1} (f_{\theta_2} \circ &g_{\theta_1}(x))-\grad_{\theta_1} (g_{\theta_2} \circ g_{\theta_1}(x'))\| \\
	=& \|(\grad_{x} f_{\theta_2} \circ g_{\theta_1}(x)) \grad_{\theta_1} g_{\theta_1}(x)-(\grad_{x} f_{\theta_2} \circ g_{\theta_1}(x')) \grad_{\theta_1} g_{\theta_1}(x')\| \\
	\leq& g(\|\theta_2\|)\|g_{\theta_1}(x)-g_{\theta_1}(x')\| g(\|\theta_1\|)\|x\| \\
	&+ g(\|\theta_2\|) g(\|\theta_1\|)\max\{\|x\|,\|x'\|\}\|x-x'\| \\
	\leq& 2g^2(\|\theta\|) \max\{\|x\|,\|x'\|\}\|x-x'\|.
\end{align*}
The second term in \cref{eq:1} is slightly simpler to bound using \cref{def:F}, and we get
\begin{align*}
	\|\grad_{\theta_2} &(f_{\theta_2} \circ g_{\theta_1}(x))-\grad_{\theta_2} (f_{\theta_2} \circ g_{\theta_1}(x'))\| = \|\grad_{\theta_2} f_{\theta_2} \circ g_{\theta_1}(x)-\grad_{\theta_2} f_{\theta_2} \circ g_{\theta_1}(x')\| \\
	\leq& g(\|\theta_2\|) \max\{\|g_{\theta_1}(x)\|,\|g_{\theta_1}(x')\|\} \|g_{\theta_1}(x)-g_{\theta_1}(x')\| \\
	\leq& g^3(\|\theta\|)\max\{\|x\|,\|x'\|\} \|x-x'\|.
\end{align*}
The last part is to prove the Lipschitz continuity of the $x$ gradient of $F_\theta$. To do this we simply note that
\begin{align*}
	\|\grad_x (f_{\theta_2} \circ &g_{\theta_1} (x))-\grad_x (f_{\theta_2} \circ g_{\theta_1} (x'))\| \\
	=& \|\grad_x f_{\theta_2} \circ g_{\theta_1} (x) \grad_x g_{\theta_1}(x)-\grad_x f_{\theta_2} \circ g_{\theta_1} (x') \grad_x g_{\theta_1}(x)\| \\
	\leq& g(\|\theta_2\|) \|g_{\theta_1} (x)-g_{\theta_1} (x')\|g(\|\theta_1\|) + g(\|\theta_2\|)g(\|\theta_1\|)\|x-x'\| \\
	\leq& g^3(\|\theta\|)\|x-x'\|.
\end{align*}

\section{SDEs and Fokker-Planck equations}
\label{sec:SDE_FP}
In this section we collect and develop the results concerning stochastic differential equations and Fokker-Planck equations that are needed in the proofs of \cref{thm:penalization,thm:modified_risk,thm:modified_risk_msq}.
\subsection{Existence, Uniqueness and Stability of SDEs}
\label{subsec:sde}
We here prove results about a type of SDEs including the ones in \cref{sgdsys}. In particular, as the risk in \cref{thm:penalization} does not satisfy the standard linear growth assumption, see \cite{GS}, we need to verify that the processes in \cref{thm:penalization} exist, that they are unique and that they satisfy a stability result w.r.t.~parameters. 

The notation in this section deviates slightly from the rest of the paper, this is by design, as the results are general in nature.
We consider
\begin{equation}
	\label{eq:SDEdiff}
	dx_t = a(t,x_t)dt + b(t,x_t) dW_t
\end{equation}
which is interpreted as the stochastic integral equation
\begin{equation}
	\label{eq:SDEint}
	x_t = x_{t_0} + \int_{t_0}^t a(t,x_t)dt + \int_{t_0}^{t} b(t,x_t) dW_t
\end{equation}
and where the second integral should be interpreted in the sense of It\^{o}. We make the following assumptions.
\begin{enumerate}[label=\textbf{A\arabic*}]
	\item \label{A1} $a = a(t,x)$ and $b = b(t,x)$ are jointly $L^2$-measurable in $(t,x) \in [t_0,T] \times \R^{d}.$
	\item \label{A2} For each compact set $\mathbf{K} \subset \R^d$ there is a Lipschitz constant $L_\mathbf{K} > 0$ such that for $x,y \in \mathbf{K}, t \in [t_0,T]$
	\begin{align*}
		\|a(t,x) - a(t,y)\| + \|b(t,x) - b(t,y)\| \leq L_\mathbf{K} \|x-y\|.
	\end{align*}
	\item \label{A3} There exists a constant $K > 0$ such that for all $(t,x) \in [t_0,T] \times \R^d$
	\begin{align*}
		a(t,x)\cdot x &\leq K^2(1+\|x\|^2),\ \|b(t,x)\|^2 \leq K^2.
	\end{align*}
	\item \label{A4} $x_{t_0}$ is independent of the Wiener process $W_t$, $t \geq t_0$, and $\E[x_{t_0}^2] < \infty$.
\end{enumerate}

\begin{lemma} \cite[Theorem 3, §6]{GS}
	If \cref{A1,A2} hold, then the solutions to the stochastic differential equations \cref{eq:SDEint} on $[t_0,T]$, corresponding to the same initial value and the same Wiener process, are path-wise unique.
\end{lemma}

\begin{theorem}
	Assuming \cref{A1,A2,A3,A4}, the stochastic differential equation \cref{eq:SDEint} has a path-wise unique strong solution $x_t$ on $[t_0,T]$ with
	\begin{equation*}
		\sup_{t_0 \leq t \leq T} \E (\|x_t\|^2) < \infty.
	\end{equation*}
\end{theorem}
\begin{proof} As we have been unable to find a reference to this result, we here outline the modifications of \cite[Theorem 3,§6, p.45]{GS} needed when dropping the linear growth assumption on $a$. 
	
	To begin, let us first remark that we can take $t_0 = 0$ without loss of generality. 
	Consider the truncation operator $\mathcal{T}^\Lambda(x) = \Lambda \frac{x}{\|x\|}$ for $\|x\| > \Lambda$ and $\mathcal{T}^\Lambda(x) = x$ otherwise.
	Then define the truncated version of \cref{eq:SDEdiff} as follows. Let $x^\Lambda_{0} = \mathcal{T}^\Lambda(x_0)$, $a^\Lambda(t,x) = a(t,\mathcal{T}^\Lambda(x))$, $b^\Lambda(t,x) = b(t,\mathcal{T}^\Lambda(x))$,
	and consider
	\begin{align}
		\label{eq:SDEtrunc}
		dx^\Lambda_t = a^\Lambda(t,x^\Lambda_t)dt + b^\Lambda(t,x^\Lambda_t) dW_t.
	\end{align}
	It is easily seen that the truncations $x^\Lambda_0, a^\Lambda, b^\Lambda$ satisfy all the requirements of \cite[Theorem 1,§6, p.40]{GS}. This gives us existence and uniqueness for the solution $x^\Lambda_t$ to \cref{eq:SDEtrunc}.
	
	In \cite[Remark 3,§6, p.48]{GS}, the authors state that the existence of a solution to \cref{eq:SDEint} can be proven with the linear growth assumption replaced with \cref{A3}. They claim that it is enough to prove that
	\begin{align} \label{eq:SDE_existence}
		\E\left [ \psi(x_0)\|x^\Lambda_t\|^2 \right ] \leq c,
	\end{align}
	where $\psi(x) = \frac{1}{1+\|x\|^2}$ and where $c$ is independent of the truncation parameter $\Lambda$.
	
	The claim \cref{eq:SDE_existence} can be proved with It\^{o}'s formula together with Gr{\"o}nwall's Lemma, see \cite{GS}. However, we also need to prove that
	\begin{align}
		\label{eq:SDE_existence2}
		\E\left [ \psi(x_0)\sup_{0 \leq t \leq T}\|x^\Lambda_t\|^2 \right ] \leq c. \quad \text{(independent of $\Lambda$)}
	\end{align}
	To do this we first apply It\^{o}'s formula to $\|x^\Lambda_t\|^2$, 
	\begin{align*}
		\|x^\Lambda_t\|^2 = \|x_0\|^2 + \int_0^t x^\Lambda_t \cdot a^\Lambda(t,x^\Lambda_t) dt +  \int_0^t x^\Lambda_s \cdot b^\Lambda(s,x^\Lambda_s) dW_s.
	\end{align*}
	Next, taking the supremum, multiplying with $\psi(x_0)$, and finally taking the expectation we get
	\begin{multline}
		\label{eq:Ito_sup}
		E\left [ \psi(x_0)\sup_{0 \leq t \leq T}\|x^\Lambda_t\|^2 \right ] 
		\leq E\left [ \psi(x_0) \|x_0\|^2 \right ]\\
		+
		E\left [ \psi(x_0)\sup_{0 \leq t \leq T} \int_0^t x^\Lambda_t \cdot a^\Lambda(t,x^\Lambda_t) dt \right ] \\
		+  E\left [ \psi(x_0)\sup_{0 \leq t \leq T} \int_0^t x^\Lambda_s \cdot b^\Lambda(s,x^\Lambda_s) dW_s \right ].
	\end{multline}
	The first term on the right in \cref{eq:Ito_sup} is bounded by $1$. We will now focus on the third term in \cref{eq:Ito_sup}. Using H{\"o}lder's inequality together with Doob's $L^p$ inequality, see for instance \cite[Theorem 1,§3, p.20]{GS}, and \cref{A3}, we get
	\begin{multline}
		\label{eq:third}
		\E\left [ \psi(x_0) \sup_{0 \leq t \leq T} \left \|\int_0^t x^\Lambda_s \cdot b^\Lambda(s,x^\Lambda_s) dW_s \right \|\right ] \\
		\leq \E\left [ \psi(x_0)^2 \sup_{0 \leq t \leq T} \left \|\int_0^t x^\Lambda_s \cdot b^\Lambda(s,x^\Lambda_s) dW_s \right \|^2 \right ] \\
		\leq C \int_0^T \E \left [ \psi(x_0)  \|x^\Lambda_s\|^2 \right ] ds,
	\end{multline}
	as $\psi^2(x) \leq \psi(x)$.
	The boundedness of the right-hand side in \cref{eq:third} follows from \cref{eq:SDE_existence}. The second term in \cref{eq:Ito_sup} can be bounded in the same way as in \cref{eq:third}, due to \cref{A3}, which concludes the proof of \cref{eq:SDE_existence2}.
	Note that the constant $c$ in \cref{eq:SDE_existence2} depends only on the structural assumptions of $a$, $b$ and not on $\Lambda$. Using \cref{eq:SDE_existence2} one can now argue as in \cite[Theorem 3, p.45]{GS} to conclude that
	\begin{align}
		\label{eq:stoppingTimeConvergence}
		\lim_{\Lambda \to \infty} \P\left ( \sup_{0 \leq t \leq T} \|x_{t}^\Lambda - x_t\| > 0 \right ) \to 0
	\end{align}
	with a rate of convergence depending only on the structural assumptions \cref{A1,A2,A3,A4}.
	We have now established all needed modifications and the rest of the existence proof follows as in \cite[Theorem 3, p.45]{GS}.
\end{proof}

\begin{lemma}[\cite{GS}, Theorem 3, Chap 2, §7] \label{thm:limitTheorem_smpl}
	
	Let $x_{n,t}$, $n = 0,1,2,\ldots$ be solutions of
	\begin{align*}
		x_{n,t} = x_0 + \int_0^t a_n(s,x_{n,s}) ds + \int_0^t b_n(s,x_{n,s})dW_s,
	\end{align*}
	where the coefficients satisfy \cref{A1,A4}, with $a$ and $b$ globally Lipschitz continuous (with a constant independent of $n$), as well as
	\begin{align*}
		\|a_n(t,x)\|^2 &\leq K^2(1+\|x\|^2) \\
		\|b_n(t,x)\|^2 &\leq K^2(1+\|x\|^2).
	\end{align*}
	If for each $\Lambda > 0$, and $s \in [0,T]$,
	\begin{align*}
		\lim_{n \to \infty} \sup_{\|x\| \leq \Lambda} \|a_n(s,x)-a_0(s,x)\| + \|b_n(s,x)-b_0(s,x)\| = 0,
	\end{align*}
	then
	\begin{align*}
		\E \left [\sup_{0 \leq t \leq T} \left \| x_{n,t} - x_{0,t} \right \|^2\right ] \to 0,
	\end{align*}
	as $n \to \infty$.
\end{lemma}

\begin{theorem} \label{thm:limitTheorem}
	Let $x_{n,t}$, $n = 0,1,2,\ldots$ be solutions of
	\begin{align*}
		x_{n,t} = x_0 + \int_0^t a_n(s,x_{n,s}) ds + \int_0^t b_n(s,x_{n,s})dW_s
	\end{align*}
	where $a_n$, $b_n$ satisfy \cref{A1,A2,A3,A4} with constants independent of $n$. If for each $\Lambda > 0$, and $s \in [0,T]$,
	\begin{align*}
		\lim_{n \to \infty} \sup_{\|x\| \leq \Lambda} \|a_n(s,x)-a_0(s,x)\| + \|b_n(s,x)-b_0(s,x)\| = 0,
	\end{align*}
	then
	\begin{align*}
		\sup_{0 \leq t \leq T} \left \| x_{n,t} - x_{0,t} \right \| \to 0
	\end{align*}
	in probability as $n \to \infty$.
\end{theorem}

\begin{proof}
	We will follow \cite[Theorem 3, §7]{GS} and modify accordingly. 
	Let $a_n^\Lambda(t,x) = a_n(t,\mathcal{T}^\Lambda(x))$, $b_n^{\Lambda}(t,x) = b_n(t,\mathcal{T}^\Lambda(x))$, $x_0^\Lambda = \mathcal{T}^\Lambda(x_0)$ and consider the solution $x_{n,t}^\Lambda$ to
	\begin{align*}
		x_{n,t}^{\Lambda} = x_0^{\Lambda} + \int_0^t a_n^{\Lambda}(s,x_{n,s}^{\Lambda}) ds + \int_0^t b_n^{\Lambda}(s,x_{n,s}^{\Lambda})dW_s.
	\end{align*}
	Following the proof of \cite[Theorem 3, §6]{GS} we first note that
	\begin{align}
		\label{eq:diffeqlarge}
		P(\sup_{0 \leq t \leq T} \|x_{n,t}^\Lambda - x_{n,t}\| > 0) \leq P(\sup_{0 \leq t \leq T} \|x_{n,t}^\Lambda\| > \Lambda).
	\end{align}
	I.e.\ as long as the truncated process $x_{n,t}^\Lambda$ stays within a ball of radius $\Lambda$ then we have $x_{n,t}^\Lambda = x_{n,t}$ a.s.: for a detailed proof see \cite[Theorem 3, §7]{GS}.
	Using \cref{eq:diffeqlarge} and the triangle inequality we see that the following holds for any $\epsilon > 0$
	\begin{multline}
		\label{eq:sdediff}
		\P(\sup_{0 \leq t \leq T} \|x_{n,t} - x_{0,t}\| > \epsilon) \leq \P(\sup_{0 \leq t \leq T} \|x_{n,t}^\Lambda - x_{0,t}\| > \epsilon) \\
		+ \P(\sup_{0 \leq t \leq T} \|x_{n,t}^\Lambda\| > \Lambda) + \P(\sup_{0 \leq t \leq T} \|x_{0,t}^\Lambda\| > \Lambda).
	\end{multline}
	The first term in \cref{eq:sdediff} goes to zero as $n \to \infty$ because our truncated equations satisfy the requirements of \cref{thm:limitTheorem_smpl}. 
	The two other summands converge to 0 as $\Lambda \to \infty$ uniformly in $n$ because the rate in \cref{eq:stoppingTimeConvergence} only depends on the constants in \cref{A1,A2,A3,A4} and the initial data, see the proof of \cref{eq:stoppingTimeConvergence}. Thus, we conclude
		\begin{align*}
			\lim_{n \to \infty} \P(\sup_{0 \leq t \leq T} \|x_{n,t} - x_{0,t}\| > 0) = 0
		\end{align*}
		which completes the proof.
\end{proof}

\subsection{Fokker-Planck equations for stochastic minimization}
\label{sub:FP_min}
In general the density of a diffusion process, for us a solution to a stochastic differential equation, satisfies a partial differential equation usually called the Fokker-Planck equation or Kolmogorov's first (forward) equation. In this section we will explore this connection for the SDEs as in \cref{sgdsys} of the form
\begin{align*}
	d \theta_t = -\grad V(\theta_t)dt + \sqrt{2} dW_t, \quad \text{in } \R^m,
\end{align*}
with initial datum $\theta_0 \sim p_0$ and for $V \in C^2(\R^m)$. Furthermore, we have for simplicity assumed $\Sigma=\sqrt{2}I_m$: the generalization to general constant full rank $\Sigma$ is straightforward.
If $-\grad V$ and $\theta_0$ satisfy \cref{A1,A2,A3,A4}, then the associated density $p_t$ for $\theta_t$ satisfies the following Cauchy problem for the Fokker-Planck equation, see for instance \cite{GS},
\begin{align}
	\label{eq:appFP}
	p_t = \div (\nabla p + p \nabla V), \quad p(x,0) = p_0, \quad \text{in $\R^m$}.
\end{align}
The equation \cref{eq:appFP} is sometimes called the \textbf{linear Fokker-Planck equation}, and for an interesting overview see for instance \cite[Chapter 8]{VBook2}.
The Cauchy problem in \cref{eq:appFP} has unique `probability' solutions in the sense of measures, i.e.\ non-negative solutions that define a probability density, see \cite{BRS,S}. Using the identity
\begin{align*}
	\div (\nabla p+p\nabla V)=\div(e^{-V}\nabla(e^{V} p))
\end{align*}
we see that we can formally rewrite \cref{eq:appFP} as
\begin{align}
	\label{eq:appFPtilde}
	\partial_t \tilde p = e^{V} \div (e^{-V} \grad \tilde p) = \Delta \tilde p - \langle \grad V , \grad \tilde p \rangle,
\end{align}
for $\tilde p = e^V p$, and where $\langle \cdot , \cdot \rangle$ denotes the dot--product in $\R^m$. Introducing the operator $L (\cdot)  = -\Delta (\cdot) + \langle \nabla V, \nabla (\cdot) \rangle$ we note that $L$ is self-adjoint in $L^2_\mu = L^2_\mu(\R^m)$ for functions in $W_\mu^{1,2} = W_\mu^{1,2}(\R^m)$, where $d\mu := e^{-V} dx$. $L^2_\mu$ and $W_\mu^{1,2}$ are the standard $L^2$ and Sobolev spaces but defined with respect to $\mu$. Indeed, integrating by parts we have
\begin{align}
	\label{eq:self_adj}
	\langle Lh, g \rangle_{L^2_\mu} = \langle \nabla h, \nabla g \rangle_{L^2_\mu},
\end{align}
where $\langle \cdot, \cdot \rangle_{L^2_\mu}$ denotes the $L^2$ inner product w.r.t.~the measure $\mu$.
Assuming that $V$ satisfies a \textbf{logarithmic-Sobolev inequality}, see \cref{subsec:logsob_hyper}, it follows, see \cite[Theorem 6]{Gross}, that $e^{tL}$ defines a \textbf{strong semi-group} which is \textbf{hyper-contractive}. Furthermore, the semi-group is positivity preserving. This establishes the existence and uniqueness of solutions to the problem in \cref{eq:appFP} for initial data in $L^2_\mu$.

The question is now, how do the solutions to \cref{eq:appFP,eq:appFPtilde} relate to each other? To answer this question, first note if $\tilde p_0 = e^V p_0 \in L^2_\mu$, $p_0 \geq 0$, then for $\tilde p = C_0^{-1} e^{tL} \tilde p_0 \geq 0$, with $C_0 = \int \tilde p_0 d\mu$, we have from \cref{eq:self_adj} that
\begin{align*}
	\partial_t \int \tilde p d\mu = \int L \tilde p \cdot 1 d\mu = \langle \grad \tilde p, 0 \rangle_{L^2_\mu} = 0 \implies \int \tilde p d\mu = 1, \quad \forall t \geq 0.
\end{align*}
Thus we have established that if the initial data $\tilde p_0$ is a probability density w.r.t.~$\mu$ then the solution $\tilde p$ is also a probability density for all $t$ w.r.t.~$\mu$. This implies that $\hat p = C_1^{-1} e^{-V} \tilde p$ with $C_1 = \int e^{-V} dx$ solves \cref{eq:appFP} and that it is a probability density w.r.t.~the Lebesgue measure. The uniqueness of `probability' solutions to \cref{eq:appFP}, again see \cite{BRS,S}, finally gives us that $\hat p = p$.

\subsection{Logarithmic Sobolev inequalities and hyper-contractivity}
\label{subsec:logsob_hyper}
The consequence of the previous subsection is that we can switch between solutions to \cref{eq:appFP} and solutions related to the semi-group $e^{tL}$ provided that the Gibbs measure $e^{-V}dx$ satisfies a \textbf{logarithmic Sobolev inequality}. In this section we discuss when the logarithmic Sobolev inequality holds. We begin with the following lemma which is a direct consequence of the Bakry-Émery theorem, see for instance \cite[Thm 21.2]{VBook}. 
\begin{lemma} \label{thm:conv_logsob}
	Assume that $V(x) \in C^2(\R^m)$ and $\grad^2 V \geq K I_m$ then $d\mu = e^{-V(x)} dx$ satisfies
	\begin{align} \label{eq:logsob}
		\int u^2 \log(u^2) d\mu - \|u\|^2_{L^2_\mu}\log(\|u\|_{L^2_\mu}) \leq \frac{2}{K} \|\grad u\|^2_{L^2_\mu}.
	\end{align}
\end{lemma}
The inequality in \cref{thm:conv_logsob} is the so called logarithmic Sobolev inequality with constant $2/K$.
Often the potential $V$ is not strictly convex, but satisfies some form of `convexity at infinity'. For such potentials the logarithmic Sobolev inequality carries over, as is made rigorous by the Holley-Stroock perturbation lemma, see for instance \cite[Prop 3.1.18]{R}. This lemma states that if we perturb the potential with a function of bounded oscillation, then the logarithmic Sobolev inequality is preserved at the expense of a larger constant.
\begin{lemma} \label{lem:perturb}
	Assume that the probability measure $\mu$ on $\R^m$ satisfies the logarithmic Sobolev inequality \cref{eq:logsob} with constant $c$. Assume that $W: \R^m \to \R$ is a bounded and measurable function. Then the modified probability measure
	\begin{align*}
		d \nu = Z^{-1}e^{-W(x)} d\mu, \qquad Z := \int e^{-W(x)} d \mu
	\end{align*}
	satisfies the log-Sobolev inequality with constant $c e^{\osc W}$, where $\osc W = \sup W - \inf W$.
\end{lemma}

Let us now consider hyper-contractivity of the semi-group $e^{tL}$.
We define the $p$ to $q$ norm of the semi-group $e^{tL}$ as
\begin{align*}
	\|e^{tL}\|_{p \to q} = \sup \{\|e^{tL} f\|_{L^q_\mu}, f \in L^2_\mu \cap L^p_\mu, \|f\|_{L^p_\mu} \leq 1\}.
\end{align*}
It then follows from a theorem of Stroock, see \cite[Section 4, Theorem 4.1, Gross]{FFGKRS} or \cite{Stroock}, that there exists, for each $p$ and $q$, where $q \geq p$, a time $t_{p\to q}$ such that if $t \geq t_{p\to q}$ then
\begin{align} \label{eq:hyper-contract}
	\|e^{t L}\|_{p \to q} \leq 1.
\end{align}
The estimate in \cref{eq:hyper-contract} is the hyper-contractivity of the semi-group $e^{tL}$.

\section{Estimates for Fokker-Planck equations} \label{sec:fp_estimates}
In the forthcoming section we prove our main results, and we will for simplicity only give the proof in the case $\Sigma=\sqrt{2}I_m$: the generalization to general constant full rank $\Sigma$ being straightforward. Consequently, we in this section consistently assume $\Sigma=\sqrt{2}I_m$. 

To prove \cref{thm:modified_risk} we will first prove perturbation estimates for the linear Fokker-Planck equation \cref{eq:appFPtilde} and then utilize this to prove convergence of the processes, $\theta^{(N)}(t)$, and the value, $\E[\tilde \Risk(\theta^{(N)}(t))]$. However, we first discuss what is a priori known about solutions to \cref{eq:appFP,eq:appFPtilde} in the sense of regularity and integrability.

We begin with the Fokker-Planck equations, see \cref{eq:appFP}, for the processes $\theta^{(N)}(t)$, $\theta(t)$,
\begin{align}
	\label{eq2a}
	\partial_tp^{(N)}&= \mbox{div}_\theta(\nabla_\theta p^{(N)}+p^{(N)}\nabla_\theta \tilde{\mathcal{R}}^{(N)}),\ p^{(N)}(\theta,0)=p_0(\theta),\notag\\
	\partial_tp&= \mbox{div}_\theta(\nabla_\theta p+p\nabla_\theta \tilde{\mathcal{R}}),\quad\quad\quad\quad\quad p(\theta,0)=p_0(\theta),
\end{align}
for $t\in (0,T]$ and where $T$ denotes the time frame used in the stochastic gradient descent. Recall that $\tilde{\mathcal{R}}^{(N)}$ and $\tilde{\mathcal{R}}$ are introduced in \cref{eq:modified_risk-}, and that $p^{(N)}$, $p$, are the probability densities to $\theta^{(N)}(t)$, $\theta(t)$, respectively. Recall from \cref{sec:SDE_FP} the equations in \cref{eq2a} have unique `probability` solutions in the sense of measures. Furthermore, since $\tilde{\mathcal{R}}^{(N)}$, $\tilde{\mathcal{R}}$ are smooth, it follows from Schauder estimates, see \cite{Lieb}, and a bootstrap argument that $p^{(N)},p$ are smooth functions of $(\theta,t)$. This implies that uniqueness holds in the point-wise sense for probability solutions.
Let us now consider the integrability properties of solutions to \cref{eq2a}. 
First note that \cref{lemmadiff2} implies that the coefficients in \cref{eq2a} satisfies all the regularity and integrability assumptions of \cite[Proposition 1, Remark 7]{LeBLi}, i.e.\
\begin{align*}
	\notag \nabla \tilde{\mathcal{R}} &\in (W^{1,1}_{\text{loc}}(\mathbb R^m)^m, \\
	\max\{0,\Delta \tilde{\mathcal{R}}\} &\in L^\infty(\mathbb R^m), \\
	\notag (1+\|\theta \|)^{-1}\nabla \tilde{\mathcal{R}} &\in (L^\infty(\mathbb R^m))^m,
\end{align*}
and the same holds true for $\tilde \Risk^{(N)}$. Thus, there exists a unique solution
\begin{align}
	\label{eq4}
	p\in L^\infty([0,T], L^2(\mathbb R^m)\cap L^\infty(\mathbb R^m)),\ \nabla p\in L^2([0,T], L^2(\mathbb R^m))
\end{align}
to the  problem for $p$ in \cref{eq2a}, in the sense that
\begin{multline}
	\label{eq2+}
	-\int\limits_0^T\int\limits_{\mathbb R^m}p\partial_t\phi\, d\theta dt-\int\limits_{\mathbb R^m}p_0\phi(\theta ,0)\, d\theta\\
 +\int\limits_0^T\int\limits_{\mathbb R^m}p\nabla \tilde{\mathcal{R}}\cdot\nabla \phi\, d\theta dt
	+\int\limits_0^T\int\limits_{\mathbb R^m}\nabla p\cdot\nabla\phi\, d\theta dt=0,
\end{multline}
for \textbf{test-functions} $\phi\in C^\infty(\mathbb R^m\times(0,T])\cap C(\mathbb R^m\times[0,T])$ with $\mbox{supp }\phi(\cdot,t)\subset K$ for all $t\in [0,T]$ and for some compact set $K\subset\mathbb R^m$. Solutions of \cref{eq2+} are called \textbf{weak solutions} and more info about weak solutions can be found in \cite{EvansBook}.
Now, as $p$, $\grad p$ are both in $L^2$ we are allowed to use $p$ as a test-function in \cref{eq2+}, resulting in the Caccioppoli estimate
\begin{align}
	\label{eq2ha-}
	\frac d{dt} \int\limits_{\mathbb R^m}p^2(\theta ,t)\, d\theta +2\int\limits_{\mathbb R^m}\|\nabla p\|^2\, d\theta =\int\limits_{\mathbb R^m}p^2\Delta \tilde{\mathcal{R}}\, d\theta \leq {\rho} \int\limits_{\mathbb R^m}p^2(\theta ,t)\, d\theta ,
\end{align}
if $\Delta \tilde{\mathcal{R}}\leq \rho$. By Gr{\"o}nwall's lemma we get from the above Caccioppoli estimate
\begin{align}
	\label{eq2ha}
	\int\limits_{\mathbb R^m}p^2(\theta ,t)\, d\theta +2\int\limits_0^T\int\limits_{\mathbb R^m}\|\nabla p(\theta ,t)\|^2\, d\theta dt\leq e^{\rho T}\int\limits_{\mathbb R^m}p_0^2\, d\theta ,
\end{align}
for $0\leq t\leq T$. \cref{eq4,eq2+,eq2ha-,eq2ha} also hold true with $p$, $\tilde{\mathcal{R}}$, replaced by $p^{(N)}$, $\tilde{\mathcal{R}}^{(N)}$. However, this type of estimate is not strong enough for our purposes.

Another approach is to switch to the semi-group $e^{tL}$ related to the linear Fokker-Planck equation \cref{eq:appFPtilde} as in \cref{sub:FP_min},
i.e.\ we consider weak solutions in spaces defined with respect to the measure $d\mu:=e^{-\tilde{\mathcal{R}}}d\theta$.
As can be easily seen, the definition of $\tilde \Risk$ implies that $\tilde \Risk$ is,  due to the truncation, a bounded perturbation of a strictly convex potential. Hence, the relations stated in \cref{sub:FP_min} hold. Therefore, we have a unique $p$ for which $\tilde p = e^{\tilde \Risk} p$ satisfies the weak formulation of \cref{eq:appFPtilde}
\begin{align}
	\label{eq2}
	- \int_0^T \int_{\R^m} \tilde p \partial_t \phi d\mu dt - \int_{\R^m} e^{\tilde \Risk} p_0 \phi(\theta,0) d\mu + \int_0^T \int_{\R^m} \grad \tilde p \cdot \grad \phi d\mu dt = 0.
\end{align}
Here $\tilde p \in L^2_\mu$ and the test-functions satisfy $\phi \in C^\infty(\mathbb R^m\times(0,T])\cap C(\mathbb R^m\times[0,T])$ and $\phi(\cdot,t) \in L^2_\mu$ for all $t \in [0,T]$. In addition, the initial data is assumed to satisfy $\tilde p_0 := e^{\tilde \Risk} p_0 \in L^2_\mu$.
In this case we deduce the a priori estimate
\begin{align}
	\label{eq2hap}
	&\int\limits_{\mathbb R^m}|e^{\tilde{\mathcal{R}}}p(\theta,t)|^2\,  d\mu +2\int\limits_0^T\int\limits_{\mathbb R^m}\|\nabla (e^{\tilde{\mathcal{R}}}p(\theta,t))\|^2\, d\mu dt\leq \int\limits_{\mathbb R^m}|e^{\tilde{\mathcal{R}}} p_0|^2\,  d\mu,
\end{align}
for $0\leq t\leq T$. Again, \cref{eq2,eq2hap} also hold true with $p$, $\tilde{\mathcal{R}}$, replaced by $p^{(N)}$, $\tilde{\mathcal{R}}^{(N)}$.
In particular, for our problems at hand we get stronger a priori estimates compared to \cref{eq2ha} by using the latter approach instead of the one in \cite{LeBLi}. 
The purpose of the section is to prove the following lemma.

\begin{lemma}
	\label{lemma1}
	Let $T>0$, $t\in [0,T]$ and consider $\tilde{\mathcal{R}}^{(N)}$, $p^{(N)}$, $\tilde{\mathcal{R}}$, $p$. Assume $p_0\in L^2(\mathbb R^m)\cap L^\infty(\mathbb R^m)$ and let $\alpha\in\mathbb R$, $\psi\in C^\infty(\mathbb R^m)$, $0\leq \psi\leq 1$. Then there exists a constant $c=c(m)$, $1\leq c<\infty$, such that the following holds
	\begin{align}\label{set 1}
		\int\limits_{\mathbb R^m}(p^{(N)})^2(\theta,t)e^{2\alpha\psi(\theta)}\, e^{\tilde{\mathcal{R}}^{(N)}}d\theta&\leq  e^{c\alpha^2\|\nabla \psi\|_\infty^2t}\int\limits_{\mathbb R^m}p^2_0(\theta)e^{2\alpha\psi(\theta)}\,  e^{\tilde{\mathcal{R}}^{(N)}}d\theta,\notag\\
		\int\limits_{\mathbb R^m}p^2(\theta,t)e^{2\alpha\psi(\theta)}\, e^{\tilde{\mathcal{R}}}d\theta&\leq  e^{c\alpha^2\|\nabla \psi\|_\infty^2t}\int\limits_{\mathbb R^m}p^2_0(\theta)e^{2\alpha\psi(\theta)}\,  e^{\tilde{\mathcal{R}}}d\theta.
	\end{align}
	Furthermore, let $E$, $F$ be compact sets of $\mathbb R^m$ such that $d(E,F)>0$, $d(E,F)$ is the distance between $E$ and $F$, and assume that $\mbox{supp } p_0\subset E$. Then
	\begin{align}\label{set 2}
		\int\limits_{F}(p^{(N)})^2(\theta,t)\, e^{\tilde{\mathcal{R}}^{(N)}}d\theta&\leq  ce^{-c(d(E,F))^2/t}\int\limits_{E}p^2_0(\theta)\, e^{\tilde{\mathcal{R}}^{(N)}}d\theta,\notag\\
		\int\limits_{F}p^2(\theta,t)\, e^{\tilde{\mathcal{R}}}d\theta&\leq  ce^{-c(d(E,F))^2/t}\int\limits_{E}p^2_0(\theta)\, e^{\tilde{\mathcal{R}}}d\theta.
	\end{align}
\end{lemma}

\begin{proof} 
	We will only supply the proof of the estimates for $p$ as the proofs in the case of $p^{(N)}$ are analogous. Recall that $p$ is smooth. Let $\alpha$ and $\psi$ be as in the statement of the lemma and let $\varphi(\theta)=e^{\alpha\psi(\theta)}$. We let $\tilde p=e^{\tilde{\mathcal{R}}}p$, and we introduce the measure $d\mu=e^{-\tilde{\mathcal{R}}}d\theta $. Using this notation, and using the test-function $p \varphi^2 \in W^{1,2}_\mu$ in \cref{eq2} together with an integration by parts, we see that
\begin{multline}
	\label{eq5g-}
	\frac 1 2\frac d{dt}\int\limits_{\mathbb R^m}\tilde p^2\varphi^2\, d\mu=-\int\limits_{\mathbb R^m}\nabla \tilde p\cdot\nabla (\tilde p\varphi^2)\, d\mu\\
	=-\int\limits_{\mathbb R^m}\|\nabla \tilde p\|^2\varphi^2\, d\mu-2\alpha\int\limits_{\mathbb R^m}\tilde p\varphi^2 \nabla \tilde p\cdot\nabla \psi\, d\mu.
\end{multline}
Using \cref{eq5g-} we first deduce that
\begin{multline}
	\label{eq5gg}
	\quad\quad\frac 1 2\frac d{dt}\int\limits_{\mathbb R^m}\tilde p^2\varphi^2\, d\mu\leq-\int\limits_{\mathbb R^m}\|\nabla \tilde p\|^2\varphi^2\, d\mu+c\alpha \|\nabla \psi\|_\infty
	\int\limits_{\R^m}\tilde p\varphi^2\|\nabla \tilde p\|\, d\mu,
\end{multline}
and then by Young's inequality, Cauchy-Schwartz and re-absorption we get
\begin{align}
	\label{eq5gg+}
	\frac d{dt}\int\limits_{\mathbb R^m}\tilde p^2\varphi^2\, d\mu&\leq c\alpha^2 \|\nabla \psi\|_\infty^2
	\int\limits_{\R^m}\tilde p^2\varphi^2\, d\mu.
\end{align}
Let $\gamma:=c\alpha^2\|\nabla \psi\|_\infty^2$ and note that \cref{eq5gg+} together with Gr{\"o}nwall's lemma give
\begin{align}
	\label{eq5+}
	\int\limits_{\mathbb R^m}\tilde p^2 e^{2\alpha\psi}\, d\mu&\leq e^{\gamma t}\int\limits_{\R^m}\tilde p^2_0 e^{2\alpha\psi}\, d\mu,
\end{align}
where $c=c(m)>0$. 
In particular going back to $p$, \cref{eq5+} becomes,
\begin{align}
\label{final}
		\int\limits_{\mathbb R^m}p^2 e^{2\alpha\psi}\, e^{\tilde{\mathcal{R}}}d\theta&\leq  e^{c\alpha^2\|\nabla \psi\|_\infty^2t}\int\limits_{\R^m}p^2_0 e^{2\alpha\psi}\,  e^{\tilde{\mathcal{R}}}d\theta,
	\end{align}
and this concludes the proof of \cref{set 1} for $p$.

To prove the estimate in \cref{set 2} for $p$, let $E$, $F$ be compact sets of $\mathbb R^m$ such that $d:=d(E,F)>0$, $d(E,F)$ is the distance between $E$ and $F$, and assume that
$\mbox{supp } p_0\subset E$. Let $\psi=1$ on $F$ and $\psi=0$ on $E$, $0\leq\psi\leq 1$. One can construct $\psi$ so that $\|\nabla \psi\|_\infty\leq cd^{-1}$, $c=c(m)>0$. Using \cref{final} we see that
\begin{align*}
		\int_{F}p^2(\theta,t)\, e^{\tilde{\mathcal{R}}}d\theta&\leq  e^{c\alpha^2d^{-2}t-2\alpha}\int\limits_{E}p^2_0(\theta)\,  e^{\tilde{\mathcal{R}}}d\theta
\end{align*}
Letting $\alpha=\beta d^2/t$ with $\beta$ so that $(c\beta-2)=-1$ we see that
$$c\alpha^2d^{-2}t-2\alpha=(c\beta-2)\beta d^2/t=-\tilde c d^2/t$$
for some $\tilde c=\tilde c(m)>0$. This concludes the proof of \cref{set 2} and the proof of the lemma.

\end{proof}

\section{Proof of the main results: \cref {thm:penalization,thm:modified_risk,thm:modified_risk_msq}}\label{sec:proof_of_main_thm}
In this section we prove \cref{thm:penalization,thm:modified_risk,thm:modified_risk_msq}, and we will for simplicity only give the proof in the case $\Sigma=\sqrt{2}I_m$.
\subsection{Proof of \cref{thm:penalization}}
Let us first prove 
\begin{align}
	\label{eq:penal_first}
	\sup_{t \in [0,T]} \|\theta_t - \theta_t^{(N)}\| \to 0, \quad \text{in probability as $N \to \infty$}.
\end{align}
We will prove \cref{eq:penal_first} under the assumption that $f_\theta \in \mathcal{A}(g)$. The first step is to show that there exists a penalization $H(\theta)$ such that
\begin{align}\label{esti}
	-\grad \Risk(\theta) \cdot \theta \leq c(1+\|\theta\|^2),
\end{align}
for some constant $c$. Let $\hat {\mathcal {R}}(\theta):=\mathcal {R}(\theta)-\gamma H(\theta)$, $\hat {\mathcal {R}}^{(N)}(\theta):={\mathcal {R}}^{(N)}(\theta)-\gamma H(\theta)$, denote the un-penalized risk. As in \cref{term2} we have
\begin{align*}
	\grad \hat {\mathcal {R}}(\theta)&= 2\E_{(x,y) \sim \mu} \left [ (y-x(1)) \grad_\theta x(1)\right ].
\end{align*}
A bound on the gradient of the risk follows from \cref{lemmadiff1},
\begin{align}
	\label{eq:riskGradBdd}
	\|\grad \hat {\mathcal {R}}(\theta)\| \leq 2\E_{(x,y) \sim \mu} \left [ (y-x(1)) \grad_\theta x(1)\right ] \leq \tilde g(g(\|\theta\|))^2 \E[\|x\|^2],
\end{align}
and the same bound holds for $\grad \hat {\mathcal {R}}^{(N)}$. Let $\phi_R$ be a cutoff function supported in $B(0,2R)$ such that $\phi_R=1$ on $B(0,R)$ and $\|\grad \phi_R\| \leq \frac{c}{R}$. Let $G(\theta)$ be a convex function satisfying $\grad G(\theta) = c\tilde g(g(\|\theta\|))^2 \theta$. With $\theta_R$ and $G$ as above we obtain from \cref{eq:riskGradBdd}
\begin{align*}
	-\grad \left ( \hat {\mathcal {R}}^{(N)}(\theta) + (1-\phi_R(\theta))G(\theta) \right ) \cdot \theta \leq c(1+\|\theta\|^2),
\end{align*}
with a constant $c(R,\mu) \geq 1$.
Letting $H(\theta):=(1-\phi_R(\theta))G(\theta)$ proves \cref{esti}. Collecting what we proved in \cref{subsec:sde}, and applying \cref{thm:limitTheorem,lemmadiff2}, we get \cref{eq:penal_first}.

We next prove the second part of \cref{thm:penalization}, i.e.\
\begin{align} \label{eq:penal_second}
	\mathbb{E} [\mathcal{R}(\theta_T)] < \infty, \quad \mathbb{E} [\mathcal{R}^{(N)}(\theta^{(N)}_T)] < \infty.
\end{align}
We will prove \cref{eq:penal_second} under the more restrictive assumption $f_\theta \in \mathcal{A}_+(g)$. We begin by showing that there is a rotationally symmetric penalization $H$ such that $\hat {\mathcal {R}} + H$ is strictly convex on $\mathbb R^m\setminus B(0,R)$ for $R\geq 1$ large. 
To do this we will bound the growth of the second derivatives of $\hat {\mathcal {R}}$. We have
\begin{align*}
	\grad^2 \hat {\mathcal {R}}(\theta) &= 2\E (\grad_\theta x_1 \otimes \grad_\theta x_1 + (y-x_1)\grad^2_\theta x_1).
\end{align*}
Using \cref{lemmadiff20,lemmagrowth} we derive the bound
\begin{align*}
	\|\grad^2 \hat {\mathcal {R}}(\theta)\| &\leq \tilde g(\|\theta\|)^2\E[\|x\|^2] + \E[(\|x\|+\|y\|)\tilde P(\|x\|)]\hat g(\|\theta\|).
\end{align*}
Hence, if   we penalize using a function $\tilde H(\theta)$ satisfying
\begin{align*}
	\grad^2 \tilde H(\theta) \geq C(g(\|\theta\|)^2+\hat g(\|\theta\|)+1) I_m,
\end{align*}
for a large enough $C(\mu) \geq 1$, then $\hat {\mathcal {R}}(\theta) + \tilde H(\theta)$ is strictly convex outside $B(0,R)$. In fact, consider a cutoff function $\phi_R$ as in the first part of the proof, and define
$V = (1-\phi_R)\hat {\mathcal {R}} + \tilde H, W = \phi_R \hat {\mathcal {R}} - \phi_R \tilde H$. Then, for a possibly slightly larger constant $C=C(R,\mu)$, we see that $V$ is strictly $K$-convex and that $W$ is bounded. Thus, $\hat {\mathcal {R}}  + (1-\phi_R) \tilde H$ is the sum of a strictly convex potential $V$ and a bounded perturbation $W$. Define the penalization as $H = (1-\phi_R)\tilde H$, then $\hat {\mathcal {R}}  + H$ is a confining potential that satisfies \cref{A1,A2,A3,A4}. From \cref{thm:conv_logsob,lem:perturb} we see that $d\mu=c e^{-(\hat \Risk+H)}d\theta$ satisfies a logarithmic Sobolev inequality and thus we can conclude the validity of the estimate in \cref{eq:hyper-contract}.

Recall from \cref{subsec:logsob_hyper} that for $V = \hat {\mathcal {R}}  + H$, we have  $\tilde p = e^{tL} c^{-1} e^V p_0 = c^{-1} e^V p$ where $p$ is the density for the stochastic process
\begin{align*}
	d\theta_t = - \grad V(\theta_t)dt + \sqrt{2} dW_t.
\end{align*}
The hyper-contractivity, i.e.\ \cref{eq:hyper-contract} shows that $\tilde p(\cdot, t) \in L^r_\mu$ for $t \geq t_r$, $r \geq 2$. This implies that
\begin{align*}
	\int \tilde p^r d\mu = \int e^{r V(\theta)} p^r(\theta) e^{-V(\theta)} d\theta = \int c^{1-r} e^{(r-1) V(\theta)}  p^r(\theta) d\theta \leq 1.
\end{align*}
Using H{\"o}lder's inequality and the above we get
\begin{align}
	\label{eq:hyper_bound}
	\notag \E[V(\theta_t)] &= \int V(\theta) p(\theta,t) d\theta
	= \int c^{\frac{r-1}{r}} e^{-\frac{r-1}{r} V} V(\theta) c^{-\frac{r-1}{r}} e^{\frac{r-1}{r} V}p(\theta,t) d\theta \\
	\notag &\leq
	\left (\int c^{1-r} e^{(r-1) V} p^r d\theta \right )^{1/r} \left (\int c e^{-\frac{r-1}{r}\frac{r}{r-1} V} V^{\frac{r}{r-1}} d\theta \right )^{\frac{r-1}{r}} \\
	&\leq \left (\int c e^{-V} V^{\frac{r}{r-1}} d\theta \right )^{\frac{r-1}{r}}.
\end{align}
The first term on the second line in \cref{eq:hyper_bound} is bounded by $1$. To bound the last term in \cref{eq:hyper_bound}, note that if $V > \frac{r}{r-1}$ then the expression $e^{-V}V^{\frac{r}{r-1}}$ decreases as $V$ increases. We can thus use the estimate
\begin{align*}
	e^{-V}V^{\frac{r}{r-1}} \lesssim e^{-K\|\theta\|^2}(K\|\theta\|^2)^{\frac{r}{r-1}}
\end{align*}
for $V > \frac{r}{r-1}$. The term on the right-hand side in the last display is integrable and thus by the Lebesgue dominated convergence theorem we see that $\E[V(\theta_t)] < \infty$. Specifically we could use $r = 2$ and thus get \cref{eq:penal_second}, completing the proof of the theorem. 
\begin{remark}
	If $t \to \infty$ then $p_\infty = c e^{-V}$ and from \cref{eq:hyper_bound} we get
	\begin{align*}
		\lim_{t \to \infty}\E[V(\theta_t)] \leq \lim_{r \to \infty} \left (\int c e^{-V} V^{\frac{r}{r-1}} d\theta \right )^{\frac{r-1}{r}} = \int c e^{-V} V d\theta = \E[V(\theta_\infty)].
	\end{align*}
\end{remark}

\subsection{Proof of \cref{thm:modified_risk}} To start the proof of \cref{thm:modified_risk} we let $\delta^{(N)}(x,t):=p(\theta,t)-p^{(N)}(\theta,t)$ and we note, as $\theta^{(N)}$ and $\theta$ share the same initial density, that $\delta^{(N)}$ satisfies the initial value problem
\begin{align} \label{eq2b}
	\partial_t\delta^{(N)}&= \mbox{div}(\nabla \delta^{(N)}+\delta^{(N)}\nabla \tilde{\mathcal{R}})-\mbox{div}(p^{(N)}\nabla (\tilde{\mathcal{R}}-\tilde{\mathcal{R}}^{(N)})),\notag\\
	\delta^{(N)}(\theta,0)&= 0.
\end{align}
Rewriting the equation \cref{eq2b} as in \cref{eq:appFPtilde}, we get
\begin{align*} 
	\partial_t(e^{\tilde{\mathcal{R}}}\delta^{(N)})e^{-\tilde{\mathcal{R}}}&= \mbox{div}(e^{-\tilde{\mathcal{R}}}\nabla(e^{\tilde{\mathcal{R}}}\delta^{(N)}))-\mbox{div}(e^{-\tilde{\mathcal{R}}}e^{\tilde{\mathcal{R}}}p^{(N)}\nabla (\tilde{\mathcal{R}}-\tilde{\mathcal{R}}^{(N)})).
\end{align*}
Now, we would like to test the above equation with $e^{\tilde \Risk} \delta^{(N)}$ but we do not a priori know that this test-function is in $L^2_\mu$, for $d\mu = e^{-\tilde \Risk} d\theta$. Instead, we use the test-function $\phi^2(\theta)e^{\tilde{\mathcal{R}}}\delta^{(N)}$, where $\phi\in C_0^\infty(\mathbb R^m)$, $0\leq \phi\leq 1$. From the above display using integration by parts we get
\begin{multline*}
	\frac{1}{2}\int\limits_{\mathbb R^m}|e^{\tilde{\mathcal{R}}}\delta^{(N)}|^2(\theta,T)\phi\, d\mu+ \int\limits_0^T\int\limits_{\mathbb R^m}\|\nabla (e^{\tilde{\mathcal{R}}}\delta^{(N)})\|^2\phi^2\, d\mu dt \\
	\begin{aligned}
		\leq&2\int\limits_0^T\int\limits_{\mathbb R^m} e^{\tilde{\mathcal{R}}} p^{(N)}
			\|\nabla (\tilde{\mathcal{R}}-\tilde{\mathcal{R}}^{(N)})\| |e^{\tilde{\mathcal{R}}} \delta^{(N)} | \|\nabla \phi\|\phi\, d\mu dt\\
		&+2\int\limits_0^T\int\limits_{\mathbb R^m}\|\nabla(e^{\tilde{\mathcal{R}}}\delta^{(N)})\| \|\nabla\phi\| |e^{\tilde{\mathcal{R}}} \delta^{(N)}|\phi\, d\mu dt\\
		&+\int\limits_0^T\int\limits_{\mathbb R^m}e^{\tilde{\mathcal{R}}} p^{(N)}\|\nabla (\tilde{\mathcal{R}}-\tilde{\mathcal{R}}^{(N)})\|\|\nabla(e^{\tilde{\mathcal{R}}}\delta^{(N)})\|
		\phi^2\, d\mu dt,
	\end{aligned}
\end{multline*}
where $d\mu = e^{-\tilde \Risk} d\theta$.
Using Youngs inequality and Cauchy-Schwarz to reabsorb terms on the left-hand side we conclude that
\begin{align}
	\label{eq:tildeRdeltaN}
	\notag\int\limits_{\mathbb R^m}|e^{\tilde{\mathcal{R}}}\delta^{(N)}|^2(\theta,T)\phi^2\, d\mu\leq&c\int\limits_0^T\int\limits_{\mathbb R^m}(e^{\tilde{\mathcal{R}}} p^{(N)})^2
	\|\nabla (\tilde{\mathcal{R}}-\tilde{\mathcal{R}}^{(N)})\|^2\phi^2\, d\mu dt\notag\\
&+c\int\limits_0^T\int\limits_{\mathbb R^m}|e^{\tilde{\mathcal{R}}} \delta^{(N)}|^2\|\nabla\phi\|^2\, d\mu dt.
\end{align}
Furthermore,
\begin{multline}
	\label{eq:tildeRdeltaN2}
	\int\limits_0^T\int\limits_{\mathbb R^m}|e^{\tilde{\mathcal{R}}} \delta^{(N)}|^2\|\nabla\phi\|^2\, d\mu dt\leq  2\int\limits_0^T\int\limits_{\mathbb R^m}p^2\|\nabla\phi\|^2e^{\tilde{\mathcal{R}}}\, d\theta dt\\
	+2\int\limits_0^T\int\limits_{\mathbb R^m}(p^{(N)})^2\|\nabla\phi\|^2e^{\tilde{\mathcal{R}}-\tilde{\mathcal{R}}^{(N)}}e^{\tilde{\mathcal{R}}^{(N)}}\, d\theta dt.
\end{multline}
To get rid of the localization that $\phi$ provides we let $\phi=\phi_R\in C_0^\infty(B(0,2R))$ be such that $\phi_R=1$ on $B(0,R)$ and $\|\grad \phi\| < \frac{C}{R}$. We will consider the limit as $R \to \infty$ in \cref{eq:tildeRdeltaN,eq:tildeRdeltaN2}. Specifically, using the estimate for $p$ stated in \cref{set 2}, \cref{lemma1}, we see that
\begin{multline}
	\label{eq:plimit}
	\int\limits_0^T \int\limits_{\R^m} p^2 \|\grad \phi\|^2 e^{\tilde \Risk} d\theta dt \leq
	\frac{c}{R^2} \int\limits_0^T \int\limits_{B(0,2R) \setminus B(0,R)} p^2 e^{\tilde \Risk} d\theta dt \\
	\leq \frac{c}{R^2} T e^{-cR^2/T} \int_{\supp(p_0)} p_0^2 e^{\tilde \Risk} d\theta \to 0
\end{multline}
as $R\to \infty$ since $d(\supp p_0, B(0,2R) \setminus B(0,R)) \approx R$ for large values of $R$. Similarly,
\begin{multline}
	\label{eq:pnbound}
\int\limits_0^T\int\limits_{\mathbb R^m}(p^{(N)})^2\|\nabla\phi\|^2 e^{\tilde{\mathcal{R}}-\tilde{\mathcal{R}}^{(N)}}e^{\tilde{\mathcal{R}}^{(N)}}\, d\theta dt \\
\leq
\frac{c}{R^2} T e^{-cR^2/T} \max_{R\leq\|\theta\|\leq 2R}e^{\tilde{\mathcal{R}}(\theta)-\tilde{\mathcal{R}}^{(N)}(\theta)} \int_{\supp(p_0)} p_0^2 e^{\tilde \Risk} d\theta.
\end{multline}
Now, recall the definitions of $\tilde {\mathcal{R}}^{(N)}$ and $\tilde {\mathcal{R}}$ introduced in \cref{eq:modified_risk-}. Using \cref{lemmadiff2}  we see that
\begin{align}\label{eq:modified_risk-c}
|\tilde {\mathcal{R}}(\theta)-\tilde {\mathcal{R}}^{(N)}(\theta)|\leq&\frac {c(g,\Lambda)}N.
\end{align}
Thus applying \cref{eq:modified_risk-c} in \cref{eq:pnbound} we get
\begin{align}
	\label{eq:pnlimit}
\int\limits_0^T\int\limits_{\mathbb R^m}(p^{(N)})^2\|\nabla\phi\|^2e^{\tilde{\mathcal{R}}-\tilde{\mathcal{R}}^{(N)}}e^{\tilde{\mathcal{R}}^{(N)}}\, d\theta dt\to 0
\end{align}
as $R\to\infty$. In particular collecting \cref{eq:tildeRdeltaN,eq:tildeRdeltaN2,eq:plimit,eq:pnlimit}, and letting $R \to \infty$ for $\phi=\phi_R$ we can conclude that
\begin{align}
	\label{eq:finaltest}
	\int\limits_{\mathbb R^m}|\delta^{(N)}|^2(\theta,T)e^{\tilde{\mathcal{R}}}\, d\theta\leq&c\int\limits_0^T\int\limits_{\mathbb R^m}(p^{(N)})^2
	\|\nabla (\tilde{\mathcal{R}}-\tilde{\mathcal{R}}^{(N)})\|^2e^{\tilde{\mathcal{R}}}\, d\theta dt.
\end{align}
Now, to bound the right-hand side of \cref{eq:finaltest} we use \cref{lemmadiff2} again to see that
\begin{align}\label{eq:modified_risk-d}
\|\nabla (\tilde{\mathcal{R}}-\tilde{\mathcal{R}}^{(N)})\|\leq&\frac {c(g,\Lambda)}N.
\end{align}
Hence, using \cref{eq:modified_risk-c} and \cref{eq:modified_risk-d}, we arrive at
\begin{multline}\label{eq5g}
\int\limits_0^T\int\limits_{\mathbb R^m}(p^{(N)})^2
	\|\nabla (\tilde{\mathcal{R}}-\tilde{\mathcal{R}}^{(N)})\|^2e^{\tilde{\mathcal{R}}}\, d\theta dt\\
\leq \frac {c(g,\Lambda)}{N^2} e^{c(g,\Lambda)/N}
\int\limits_0^T\int\limits_{\mathbb R^m}(p^{(N)})^2e^{\tilde{\mathcal{R}}^{(N)}}\, d\theta dt.
\end{multline}
We can apply the estimate for $p^{(N)}$ stated in \cref{set 1}, see \cref{lemma1}, and we deduce that
\begin{align}\label{eq5g+}
\int\limits_0^T\int\limits_{\mathbb R^m}( p^{(N)})^2e^{\tilde{\mathcal{R}}^{(N)}}\, d\theta dt\leq& c(T)\int\limits_{B(0,R_0)}p_0^2e^{\tilde{\mathcal{R}}^{(N)}}\, d\theta.
\end{align}
Now collecting \cref{eq:finaltest,eq5g,eq5g+} we get
\begin{multline}
\int\limits_0^T\int\limits_{\mathbb R^m}(p^{(N)})^2
	\|\nabla (\tilde{\mathcal{R}}-\tilde{\mathcal{R}}^{(N)})\|^2e^{\tilde{\mathcal{R}}}\, d\theta dt\\
\leq \frac {c(g,\Lambda, T)}{N^2} e^{c(g,\Lambda,R_0)/N}\int\limits_{B(0,R_0)}p_0^2\, d\theta,
\end{multline}
and finally
\begin{align}\label{eq5g+++}
	\int\limits_{\mathbb R^m}|\delta^{(N)}|^2(\theta,T)e^{\tilde{\mathcal{R}}}\, d\theta\leq&\frac {c(g,\Lambda, T)}{N^2} e^{c(g,\Lambda,R_0)/N}\int\limits_{B(0,R_0)}p_0^2\, d\theta.
\end{align}
Using Cauchy-Schwartz and \cref{eq5g+++}, we get
\begin{multline}
	\label{eq5ga}
	\int\limits_{\mathbb R^m}e^{\tilde{\mathcal{R}}/4}| \delta^{(N)}|(\theta,T)\, d\theta \\
	\leq \left [\int_{\mathbb R^m}e^{\tilde{\mathcal{R}}}| \delta^{(N)}|^2(\theta,T)\, d\theta \right ] ^{\frac{1}{2}}\left [\int_{\mathbb R^m}e^{-\tilde{\mathcal{R}}/2}\, d\theta \right ]^{\frac{1}{2}} \\
	\leq c(m,\gamma,g,\Lambda,R_0, T) N^{-1}\|p_0\|_2.
\end{multline}
We will now use the estimate in \cref{eq5ga} to conclude the proof of \cref{thm:modified_risk}. Indeed, by \cref{eq5ga}
\begin{align}
	\label{eq:final_first}
	\bigl \|\E[\theta(T)-\theta^{(N)}(T)]\bigr \|&=\left \|\int_{\mathbb R^m}\theta (p(\theta,T)- p^{(N)}(\theta,T))\, d\theta\right \|\notag\\
&\leq \int\limits_{\mathbb R^m}\|\theta\| |\delta^{(N)}(\theta)|\, d\theta\notag\\
&=\int\limits_{\mathbb R^m}\|\theta\| e^{- \tilde{\mathcal{R}}/4} |\delta^{(N)}(\theta)|e^{\tilde{\mathcal{R}}/4}\, d\theta\notag\\
&\leq \sup_{\|\theta\|} \|\theta\| e^{- \tilde{\mathcal{R}}/4} c(m,\gamma,g,\Lambda, R_0, T) N^{-1}\|p_0\|_2.
\end{align}
Now by the definition of $\tilde \Risk$ as the truncated risk + a quadratic penalization we see that $\tilde \Risk \approx \|\theta\|^2$ for large $\|\theta\|$ and thus $\sup_{\|\theta\|} \|\theta\| e^{- \tilde{\mathcal{R}}/4} \leq c(g,\Lambda,\lambda,\rho_0)$. This together with \cref{eq:final_first} concludes the proof of \cref{thm2.4a} in \cref{thm:modified_risk}. To prove the other estimate, \cref{thm2.4b}, we note that
\begin{multline}
	\label{eq:valuebound}
	\bigl |\E [\tilde{\mathcal{R}}(\theta(T))-\tilde{\mathcal{R}}^{(N)}(\theta^{(N)}(T)) ]\bigr |\leq  \int\limits_{\mathbb R^m}\tilde{\mathcal{R}}(\theta)|p(\theta,T)-p^{(N)}(\theta,T)|\, d\theta\\
	+\int\limits_{\mathbb R^m}p^{(N)}(\theta,T)|\tilde{\mathcal{R}}(\theta)-\tilde{\mathcal{R}}^{(N)}(\theta)|\, d\theta.
\end{multline}
Again using \cref{eq5ga} we see that
\begin{align}\label{esta}
	\int\limits_{\mathbb R^m}\tilde{\mathcal{R}}(\theta)|p(\theta,T)-p^{(N)}(\theta,T)|\, d\theta\leq  c(m,\gamma,g,\Lambda, \lambda,\rho_0, R_0, T) N^{-1}\|p_0\|_2.
\end{align}
Furthermore, using \cref{eq:modified_risk-c}, we see that
\begin{align}
	\label{eq:pRRN}
	\int\limits_{\mathbb R^m}p^{(N)}(\theta,T)|\tilde{\mathcal{R}}(\theta)-\tilde{\mathcal{R}}^{(N)}(\theta)|\, d\theta\leq \frac{c}{N}.
\end{align}
Collecting \cref{eq:valuebound,esta,eq:pRRN} completes the proof of \cref{thm2.4b} in \cref{thm:modified_risk}.
To prove the second set of estimates we simply have to repeat the argument
and use the estimates stated in \cref{set 2}, \cref{lemma1}.

\subsection{Proof of \cref{thm:modified_risk_msq}} The result follows immediately from the  fact that based on  the truncation procedure $T_\Lambda$, and  \cref{lemmadiff1,lemmadiff2}, we can conclude that the processes $\theta$ and $\theta^{(N)}$ both satisfy the assumptions of \cref{thm:limitTheorem_smpl}, and this proves our theorem.

\section{Numerical experiments}\label{sec:numerical}
To investigate how the optimized networks depend on the number of residual connections, i.e.\ layers, we have conducted a number of numerical simulations. The idea is simply to optimize \cref{resnet3} for different values of $N$. To do so we consider a few different topologies and problems. We have conducted our experiments on the following datasets:
\begin{enumerate}
	\item Cifar10 \cite{KH}
	\item Starlike Annuli + Starlike Annuli with augmentation.
\end{enumerate}

\subsection{Cifar10}
Cifar10 is a well-known dataset consisting of images of 10 classes, [airplanes, cars, birds, cats, deer, dogs, frogs, horses, ships, and trucks]. The task is to take a 32x32 image as input and as output the class of the image. A lot of research papers have proposed different topologies for Cifar10 and currently the accuracy of the best performing model is $96.53\%$ \cite{Graham}. However, the model closest to our experiments is the Residual Network (ResNet) \cite{He} based on which they obtained an accuracy of $\approx 93.4\%$. For comparison, the average human performance is $93.91\pm 1.52 \%$ \cite{H}. From our perspective, the interesting point about this dataset is that it is the simplest standard dataset where convolutional networks vastly outperform other models, like for instance fully connected neural networks. For this dataset we have designed a simple network that performs fairly well compared to state of the art mentioned above. 

\subsubsection*{\textbf{Network topology}}
To continue let us describe the network topology used in our experiments, see \cref{figCNN} for a schematic. Since networks of the type in \cref{resnet3} need to have the same input-dimension and output-dimension, channels, we first apply a convolutional layer with $3 \times 3$ kernels and $256$ output channels with a stride of $2$. This is then the input to the iterated block, of type \cref{resnet3}, that has $256$ input channels and $256$ output channels. The block is applied $N$ times. After this we apply a batch normalization layer, \cite{IoffeSzeg}, as well as a $3 \times 3$ kernel layer with $64$ channels with a stride of $4$ and ReLU activation. This is then followed by another batch normalization layer, a 4x4 average pooling layer and finally a softmax layer. 
The network in the main iterated block has the type \cref{nnspec}. Specifically, it consists of a $3 \times 3$ kernel layer with $256$ output channels and ReLU activation. This feeds into another $3 \times 3$ kernel layer with again $256$ output channels, with no activation. Although this network topology is quite simple the layers are fairly wide, and we end up with a total number of about 1M parameters.
\subsubsection*{\textbf{Experimental setup}}
The experimental setup is as follows. The Cifar10 dataset consists of 50k images in the training set and 10k images in the test set. In order to get a better estimate on the generalization performance, we performed 10-fold cross validation on the training set, for $N = [1,2,3,4,5,9,15,20,30,100]$.
We use a weight decay of $10^{-6}$, a momentum of $0.9$, and we use the weight initialization from \cite{Glorot}. The models are trained with a batch size of $128$ on a single GPU (RTX 2080ti) for $N < 10$ and a batch size $512$ split over $4$ GPUs (RTX 2080ti) for $N \geq 10$. We start with a learning rate of $0.1$, divide it by $10$ at $80$, $120$, $160$ and by $2$ at $180$ epochs, where an epoch consists of 45k/batchsize iterations. Training is terminated at $200$ epochs. We use the simple data augmentation depicted in \cite{Lee} when training: i.e.\ $3$ pixels are padded on each side, and a $32\times32$ crop is randomly sampled from the padded image or its horizontal flip. For testing, we only evaluate the single view of the original $32\times32$ image.
All images are normalized by the pixel mean.

\subsubsection*{\textbf{Results}}
The results are depicted in \cref{figCNNResult}, and we note that our best model gives us $90.5\%$ which is good for such a simple topology. It is interesting to note that the cross validation accuracy is increasing with layer count, and we saw a similar increase in the training accuracy. As the parameter count stays the same with increasing depth this means that for this problem deeper is better.
Also, note that the accuracy flattens out fairly quickly. This can be interpreted as to indicate that there is limited value in using too many layers for this type of image classification problems. It may be worthwhile to implement a limit on the number of layers in the methods developed in \cite{CRBD}.
Another observation is that the speed of convergence for the validation accuracy is close to $1/N$, albeit slightly faster. We believe that the convergence rate for the risk (or related metric) will in practice almost always be faster than our theoretical bound of $1/N$, see \cref{thm:modified_risk}.

\begin{figure}
	\includegraphics[width=0.5\textwidth]{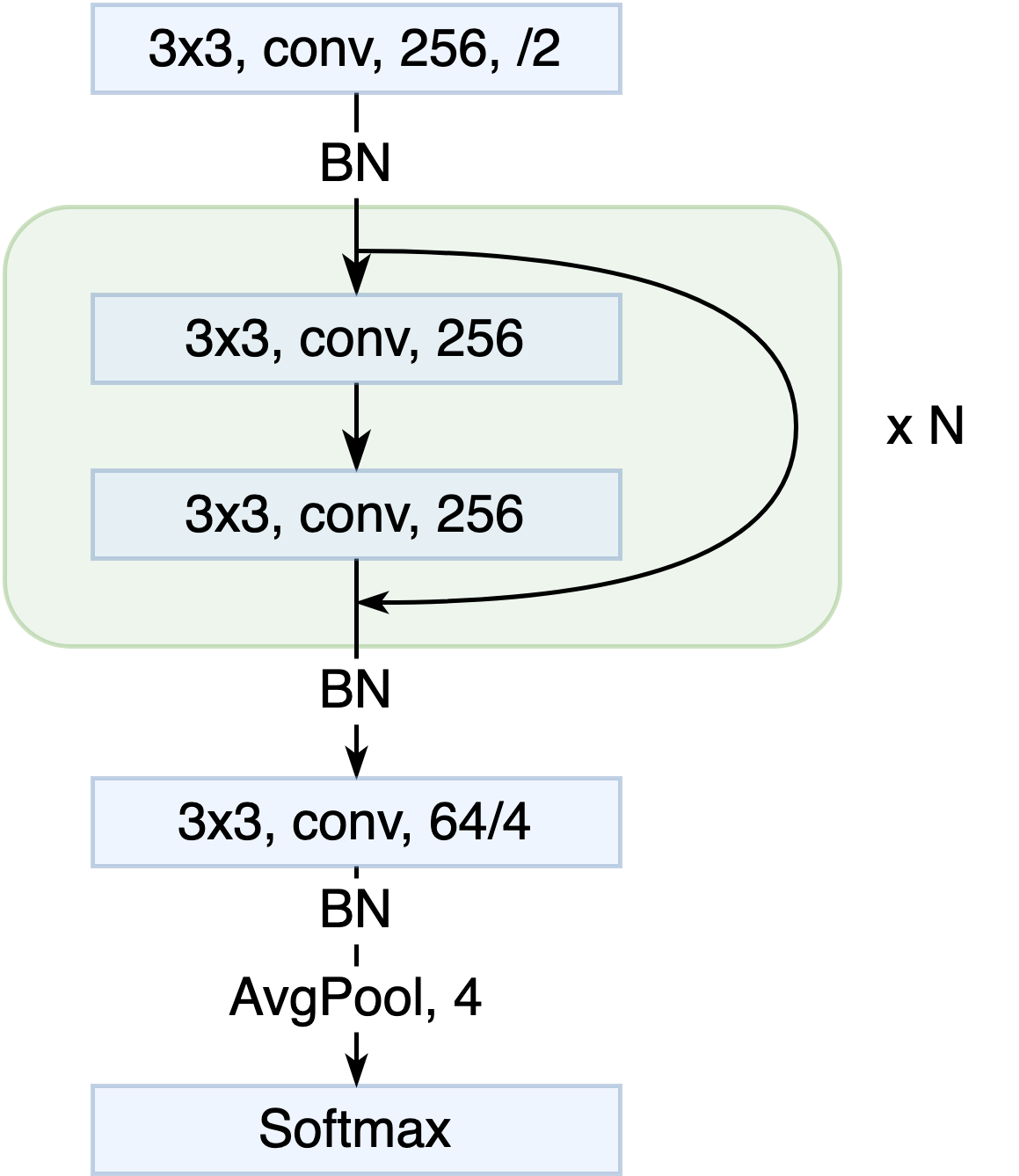}
	\caption{CNN Topology for Cifar10}
	\label{figCNN}
\end{figure}
\begin{figure}
	\includegraphics[width=\textwidth]{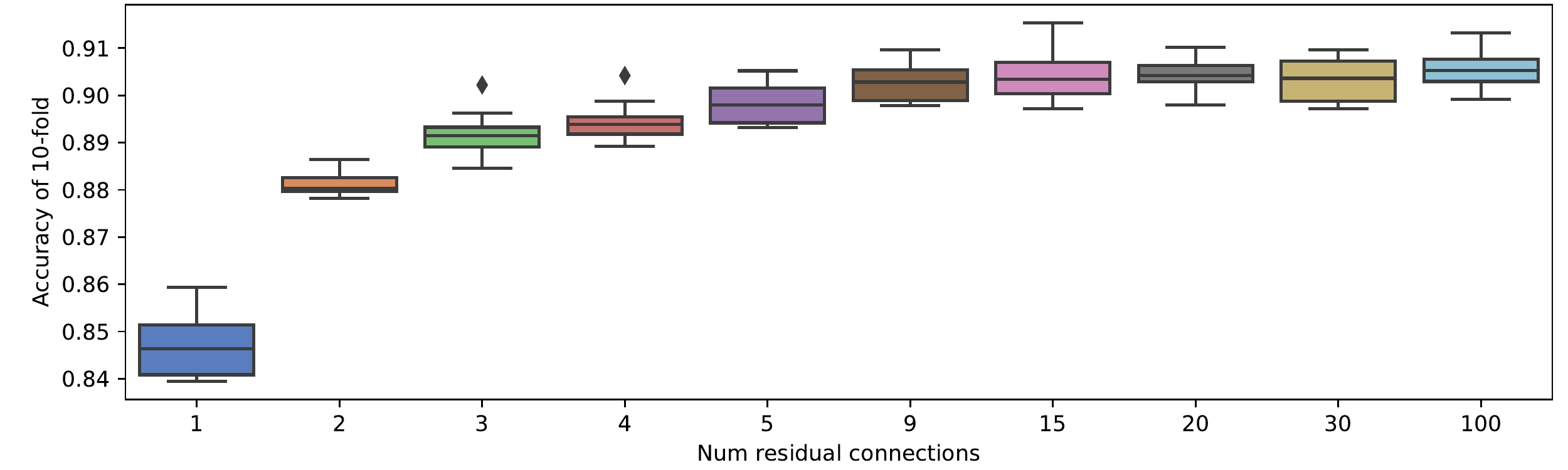}
	\caption{Standard Box-plot of 10 fold cross validation result on Cifar10 with different number of residual connections using the topology in \cref{figCNN}.}
	\label{figCNNResult}
\end{figure}

\subsection{Starlike Annuli dataset}
An interesting observation that was made in \cite{DDW}, is that certain simple functions cannot be represented as a neural ODE of the type defined in \cref{neuralODE--}. In \cite{DDW} the authors give the following example. Let $0 < r_1 < r_2 < r_3$ and let $g:\R^d \to \R$ be such that
\begin{align}
	\label{eq:neuralODEBad}
	\begin{cases}
		g(x) = -1, & \text{if $\|x\|\leq r_1$} \\
		g(x) = 1, & \text{if $r_2 \leq \|x\|\leq r_3$}.
	\end{cases}
\end{align}
To attempt to represent $g$ using a neural ODE $h_t$ as in \cref{neuralODE---} we compose the output of the neural ODE $h_1(x)$ with a linear map $\mathcal{L}:\R^d \to \R$. In \cite{DDW} they prove that $h_1(x)$ is a homeomorphism.
In order for $\mathcal{L}(h_1(x)) = g(x)$ to hold, the neural ODE $h_1(x)$, which is a homeomporhism, have to map the ball $B_{r_1}$ to one side of a hyper-plane and the outer annuli to the other side, which is impossible. For the numerical simulations in \cite{DDW} the authors consider $g$ as defining a binary classification problem. The data-points are sampled uniformly at random from the two classes defined by $g$ with a 50/50 class balance. When solving this binary classification problem using the neural ODE model, the authors in \cite{DDW} noted that the corresponding ODEs are numerically hard to solve and require many costly function evaluations. They propose a solution to this problem by simply embedding the dataset in a bigger space, for instance $\R^d \to \R^{d+1}$. This solves the problem of representability, in the sense above, and the corresponding classification problem becomes easy to solve using Neural ODEs. The process of embedding the dataset in a bigger space is called augmentation.

\subsubsection*{\textbf{Dataset}}
Our purpose is to compare the convergence rate of the risk for non-representable problems vs representable problems, in the sense above. As such, a comparison between \cref{eq:neuralODEBad} and its augmentation makes sense. However, the augmented version of \cref{eq:neuralODEBad} is very easy as it is essentially a conical solution, which is easy to represent with a neural network using ReLU activation functions. We therefore consider a related dataset that is harder than \cref{eq:neuralODEBad} even in the augmented case, we call this dataset the Starlike Annuli. To describe our construction, see \cref{figStar}, we let $r_1 < r_2 < r_3$ and define, using polar coordinates,
\begin{equation} \label{eq:starlikeAnnuli}
	g(r,\theta) =
	\begin{cases}
		-1, & \text{ if } r \leq r_1 (2+\cos(5\theta)) \\
		1, & \text{ if } r_2 (2+\cos(5\theta)) \leq r \leq r_3 (2+\cos(5\theta)). \\
	\end{cases}
\end{equation}
Here the difference $r_3-r_1$ denotes the class separation. The data is generated by sampling uniformly at random from each of the two class regions. The augmented version is the dataset embedded in 3 dimensions. For our experiments we chose $r_1 = 1$, $r_2 = 1.5$, and $r_3=3$.

\subsubsection*{\textbf{Network topology}}
The topology of our network is \cref{nnspec} and iterated as in \cref{resnet3}, i.e.\ 
\begin{align*}
	f_\theta(x) = K^{(1)} \sigma (K^{(2)} x + b^{(2)}) + b^{(1)}
\end{align*}
where $K^{(2)} \in \R^{m} \times \R^{2+n}$, $b^{(2)} \in \R^m$, $K^{(1)} \in \R^{2+n} \times \R^m$, $b^{(1)} \in \R^{2+n}$ and $\sigma$ is the ReLU activation function.
Following \cite{DDW},  $n \geq 0$ is the number of augmented dimensions, $m > 0$ is the size of the inner layer. In our experiments we chose $m=16$ and $n=0,1$. At the end of the network we have a softmax layer. We are using a quadratic penalization, as in the theoretical result \cref{thm:modified_risk}, with constant $\gamma = 0.001, \lambda=1$. 

\subsubsection*{\textbf{Experimental setup}}
The experimental setup is as follows. The Starlike Annuli dataset consists of 50k points drawn uniformly at random from \cref{eq:starlikeAnnuli} with a 50/50 class balance. As in the Cifar10 experiment, we estimate generalization by performing 10-fold cross validation on the training set itself, for $N = [1,2,5,10,20,100,200]$.
We use a weight decay of $0$ and a momentum of $0.9$, and we use the weight initialization in \cite{Glorot}. These models are trained with a batch size of $128$ on a single GPU (RTX 2080ti). We start with a learning rate of $0.1$, divide it by $10$ at $80$, $120$, $160$ and by $2$ at $180$ epochs where an epoch consists of 45k/batchsize. Training is terminated at $200$ epochs. We ran one experiment with zero augmented dimensions and one experiment with one augmented dimension.

\subsubsection*{\textbf{Results}}
The results are depicted in \cref{figAnnulus,figAnnulusAug}. We chose to display the Cross Entropy instead of accuracy for visual clarity. In the non-augmented case, results displayed in \cref{figAnnulus}, it is interesting to note a few things. First of all the Cross-Entropy (Risk) is increasing with the layer count, indicating that deeper is worse for this problem. This is reasonable to expect as \cref{eq:starlikeAnnuli} is of the same type as \cref{eq:neuralODEBad} for the limit problem (NODE). Thus, it is non-representable, in the sense stated previously, within the NODE function class, \cite{DDW}. Secondly the convergence rate of the risk is indeed very close to the theoretical bound $1/N$. 

On the other hand, the one augmented dimension problem, see \cref{figAnnulusAug}, tells a different story. For $N \geq 2$ we see no noticeable difference in performance, and we seem to gain nothing from the added depth. We interpret this as to indicate that the convergence rate is very fast and that we essentially reach the NODE performance using only a small value of $N$.

\subsection{Conclusion}

It is interesting to observe that for some problems the performance either increase or stay the same as depth increase, see \cref{figCNNResult,figAnnulusAug}, while for others the performance decrease significantly, see \cref{figAnnulus}.

What is the reason for this? Looking at this from the context of regularity, we note that the deeper the network is, the closer we are to the NODE, \cref{thm:modified_risk}. This forces the trajectories to become smoother and smoother, and we see that depth is a regularizer for the learning problem. On one hand, the idea of regularization is supported by the observations above, especially \cref{figAnnulus}. On the other hand, for the experiments regarding Cifar10 we saw a similar increase in training accuracy for increasing depth, as we saw in the test accuracy, \cref{figCNNResult}. This contradicts the idea of regularization. Putting these observations together we see that regularization alone cannot explain our numerical simulations. We believe that there is likely a complicated interplay between the effect that the depth has on the `loss landscape', \cite{LiLoss}, and it's regularizing effects.

\begin{figure}
	\includegraphics[width=0.7\textwidth]{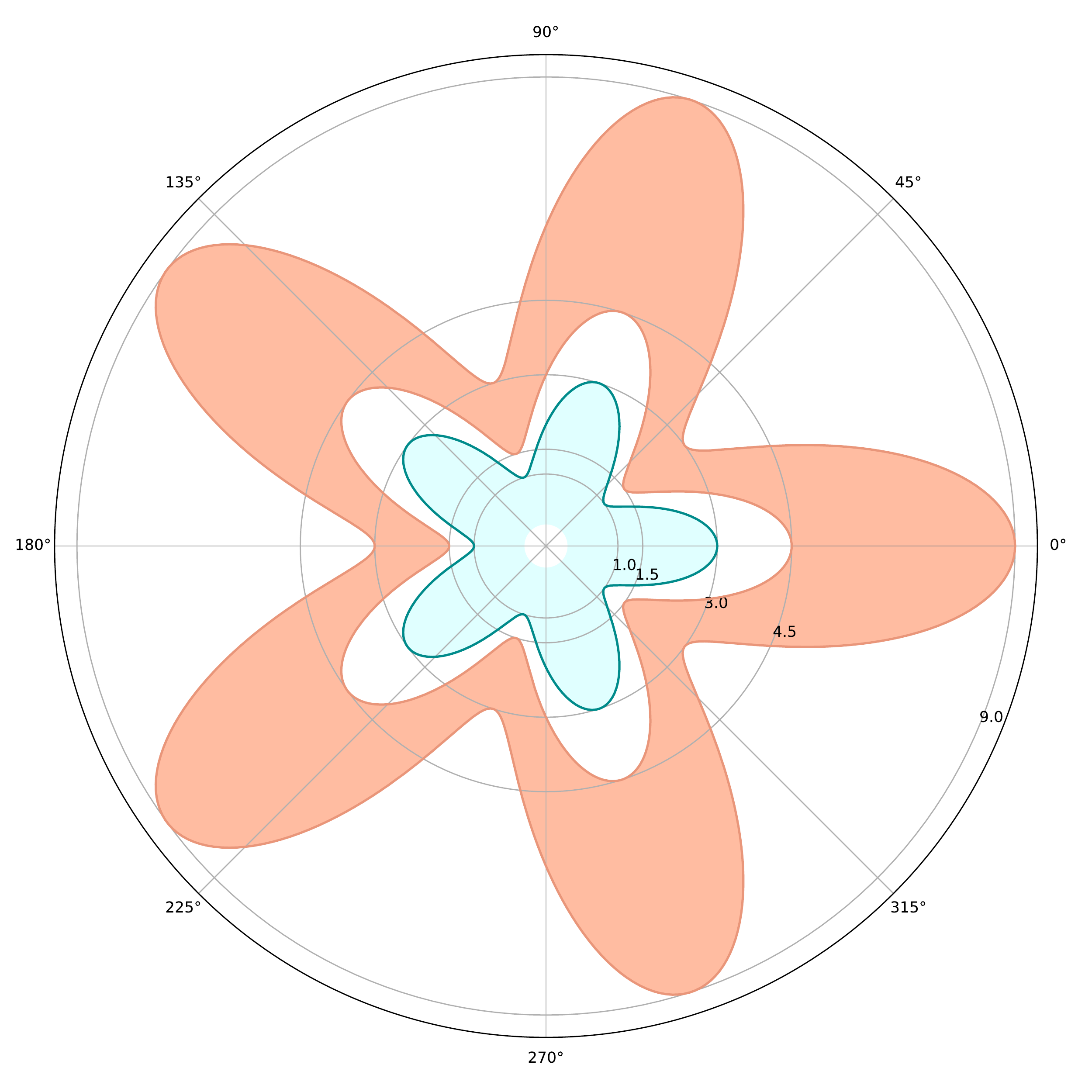}
	\caption{Annuli dataset}
	\label{figStar}
\end{figure}
\begin{figure}
	\includegraphics[width=\textwidth]{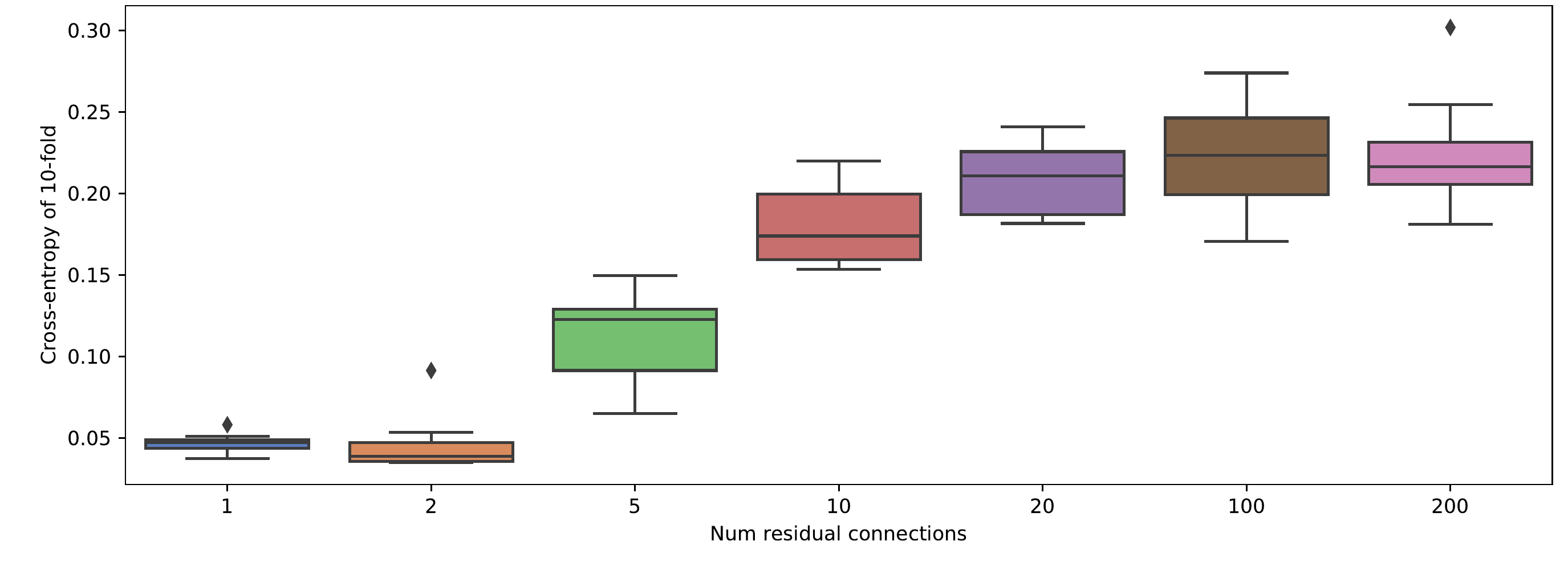}
	\caption{10 fold cross validation result on the starlike annulus dataset}
	\label{figAnnulus}
\end{figure}
\begin{figure}
	\includegraphics[width=\textwidth]{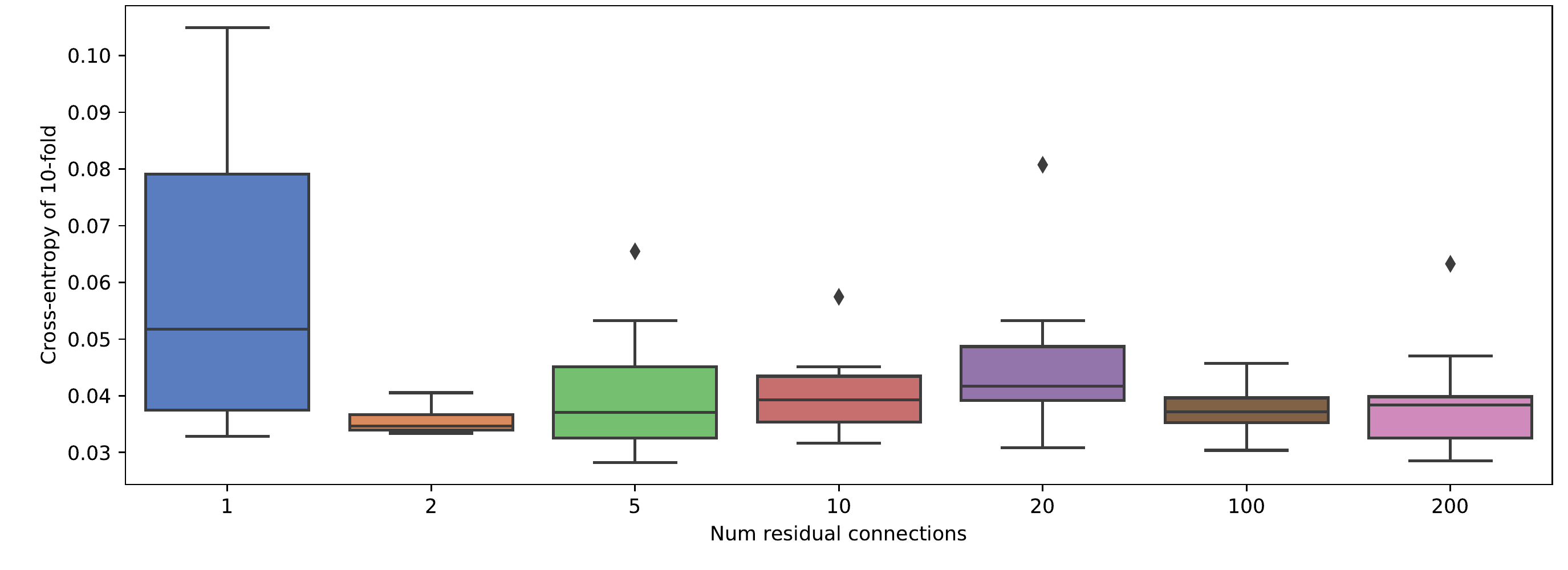}
	\caption{10 fold cross validation result on the annulus dataset with 1 augmented dimension}
	\label{figAnnulusAug}
\end{figure}

\section{Concluding remarks}\label{sec:conclusion}

In this paper we have established several convergence results, \cref{thm:penalization,thm:modified_risk,thm:modified_risk_msq}, which together give a theoretical foundation for considering Neural ODEs (NODEs) as the deep limit of
ResNets. The rate of convergence proved is essentially sharp as it is the order of the discretization error for the ODE. Our results can be seen as a complement to the result in \cite{TVG}, where the authors consider a related but different problem. However, in \cite{TVG} the authors only consider the convergence of
minimizers and not the convergence of the actual optimization procedure. They conjecture that minimizers also converge with the rate $1/N$, which is what we prove for the stochastic gradient descent. It should also be noted that we in this paper assume that the gradient of the NODE can be computed exactly. This is in contrast to \cite{CRBD} where they used an adjoint method to compute an
approximate gradient. On the other hand, \cite{CRBD} contains no theoretical justifications. The error produced when applying the adjoint method to compute the gradient is not known and it is hard to tell how it would affect the convergence considered in \cref{thm:modified_risk}.

\cref{sec:SDE_FP} is of independent interest, and while the results of the section may be `folklore' we have provided this section as it rigorously establishes the existence and uniqueness of stochastic gradient descent as well as the connection to the corresponding Fokker-Planck and linear Fokker-Planck equation. Using what is established in \cref{sec:SDE_FP} we can also make the observation that, instead of considering the rate of convergence to minima of $V$ for
\begin{align*}
	d \theta_t = -\grad V(\theta_t) dt + \sqrt{2} dW_t,
\end{align*}
we can consider the convergence of the density $p_t$ of $\theta_t$
\begin{align*}
	p_t \to p_\infty(\theta) = c e^{-V}, \quad t \to \infty.
\end{align*}
Using \cref{eq:appFPtilde} we see that the corresponding solution $\tilde p = e^V p$ satisfies
\begin{align*}
	\partial_t (\tilde p_t-1) = L(\tilde p_t - 1).
\end{align*}
According to the Poincaré inequality for measures that satisfy logarithmic Sobolev inequalities, see \cite{VBook2},
\begin{align*}
	\frac{1}{2} \partial_t \|\tilde p_t-1 \|^2_{L_\mu^2} &= \langle \partial_t (\tilde p_t-1), (\tilde p_t-1) \rangle_{L_\mu^2} \\
	&= -\langle L(\tilde p_t - 1), (\tilde p_t - 1) \rangle_{L_\mu^2} = -\|\grad (\tilde p_t - 1)\|^2_{L_\mu^2} \\
	&\leq - c \|\tilde p_t - 1\|^2_{L_\mu^2},
\end{align*}
using \cref{eq:self_adj}. Now, by Gr{\"o}nwall's lemma we see that
\begin{align*}
	\|\tilde p_t-1 \|^2_{L_\mu^2} \leq e^{-ct} \|\tilde p_0-1 \|^2_{L_\mu^2}.
\end{align*}
The conclusion is that $p_t \to p_\infty$ exponentially fast as $t \to \infty$.
So instead of minimizing the possibly non-convex function $V$ we see that stochastic gradient descent is actually minimizing a convex problem in the sense of distributions. This could explain why stochastic gradient descent performs so well in practice.

In all our proofs, and in fact even for the existence of strong solutions to the stochastic gradient descent, \cref{sgdsys}, we need a penalization term. In particular, this is the case for neural networks with more than one layer, in this case we cannot simply guarantee from the neural network that a condition such as \cref{A3} holds. Instead, we have to rely on bounds concerning the absolute value of the drift, which grows polynomially, the order of the polynomial being the number of layers. For the Neural ODE it is clear that for the problem in \cref{eq:neuralODEBad} one would need $\|\theta\| \to \infty$ in order to produce a zero risk, and hence the penalization is necessary. However, the practical implications of such a penalization is essentially non-existent. This is because we can choose a penalization that is zero on a ball of arbitrary size, for instance $B(0,1.8 \cdot 10^{308})$, which is the largest value for an IEEE \textbf{754} 64--bit double.

\vspace{0.5cm}

\noindent {\bf Acknowledgment} The authors were partially supported by a grant from the G{\"o}ran Gustafsson Foundation for Research in Natural Sciences and Medicine.

\bibliographystyle{acm}
\bibliography{references}

\begin{thebibliography}{10}

\bibitem{BRS}
{\sc Bogachev, V., R{\"o}ckner, M., and Shaposhnikov, S.}
\newblock On uniqueness problems related to the {F}okker-{P}lanck-{K}olmogorov
  equation for measures.
\newblock {\em Journal of Mathematical Sciences 179}, 1 (2011), 7--47.

\bibitem{BGMSS}
{\sc Brutzkus, A., Globerson, A., Malach, E., and Shalev-Shwartz, S.}
\newblock {SGD} learns over-parameterized networks that provably generalize on
  linearly separable data.
\newblock https://arxiv.org/abs/1710.10174, 2017.

\bibitem{COOSC}
{\sc Chaudhari, P., Oberman, A., Osher, S., Soatto, S., and Carlier, G.}
\newblock Deep relaxation: partial differential equations for optimizing deep
  neural networks.
\newblock {\em Research in the Mathematical Sciences 5}, 3 (Jun 2018), 30.

\bibitem{CS}
{\sc Chaudhari, P., and Soatto, S.}
\newblock Stochastic gradient descent performs variational inference, converges
  to limit cycles for deep networks.
\newblock {\em 2018 Information Theory and Applications Workshop, ITA 2018\/}
  (2018).

\bibitem{CRBD}
{\sc Chen, T.~Q., Rubanova, Y., Bettencourt, J., and Duvenaud, D.~K.}
\newblock Neural ordinary differential equations.
\newblock In {\em Advances in Neural Information Processing Systems 31},
  S.~Bengio, H.~Wallach, H.~Larochelle, K.~Grauman, N.~Cesa-Bianchi, and
  R.~Garnett, Eds. Curran Associates, Inc., 2018, pp.~6571--6583.

\bibitem{CLTZ}
{\sc Chen, X., Lee, J.~D., Tong, X.~T., and Zhang, Y.}
\newblock Statistical inference for model parameters in {S}tochastic {G}radient
  {D}escent.
\newblock https://arxiv.org/abs/1610.08637.

\bibitem{Maso}
{\sc Dal~Maso, G.}
\newblock {\em An Introduction to $\Gamma$-convergence}.
\newblock Progress in nonlinear differential equations and their applications.
  Birkh{\"a}user, 1993.

\bibitem{DuLee}
{\sc Du, S., and Lee, J.}
\newblock On the power of over-parametrization in neural networks with
  quadratic activation.
\newblock {\em 35th International Conference on Machine Learning, ICML 2018
  3\/} (2018), 2132--2141.

\bibitem{DDW}
{\sc Dupont, E., Doucet, A., and Teh, Y.~W.}
\newblock Augmented {N}eural {ODEs}.
\newblock https://arxiv.org/abs/1904.01681.

\bibitem{EHL}
{\sc E, W., Han, J., and Li, Q.}
\newblock A mean-field optimal control formulation of deep learning.
\newblock {\em Research in the Mathematical Sciences 6}, 1 (Dec 2018).

\bibitem{EvansBook}
{\sc Evans, L.~C.}
\newblock {\em Partial differential equations}, second~ed., vol.~19 of {\em
  Graduate Studies in Mathematics}.
\newblock American Mathematical Society, Providence, RI, 2010.

\bibitem{FFGKRS}
{\sc Fabes, E., Fukushima, M., Gross, L., Kenig, C., R\"{o}ckner, M., and
  Stroock, D.~W.}
\newblock {\em Dirichlet forms}, vol.~1563 of {\em Lecture Notes in
  Mathematics}.
\newblock Springer-Verlag, Berlin, 1993.
\newblock Lectures given at the First C.I.M.E. Session held in Varenna, June
  8--19, 1992, Edited by G. Dell'Antonio and U. Mosco.

\bibitem{GL}
{\sc Ghadimi, S., and Lan, G.}
\newblock Stochastic first- and zeroth-order methods for nonconvex stochastic
  programming.
\newblock {\em SIAM Journal on Optimization 23}, 4 (2013), 2341--2368.

\bibitem{GS}
{\sc G\={\i}hman, u.~I., and Skorohod, A.~V.}
\newblock {\em Stochastic differential equations}.
\newblock Springer-Verlag, New York-Heidelberg, 1972.
\newblock Translated from the Russian by Kenneth Wickwire, Ergebnisse der
  Mathematik und ihrer Grenzgebiete, Band 72.

\bibitem{Glorot}
{\sc Glorot, X., and Bengio, Y.}
\newblock Understanding the difficulty of training deep feedforward neural
  networks.
\newblock {\em Journal of Machine Learning Research 9\/} (2010), 249--256.

\bibitem{Graham}
{\sc Graham, B.}
\newblock Fractional max-pooling.
\newblock {\em CoRR abs/1412.6071\/} (2014).

\bibitem{Gross}
{\sc Gross, L.}
\newblock Logarithmic {S}obolev inequalities.
\newblock {\em Amer. J. Math. 97}, 4 (1975), 1061--1083.

\bibitem{He}
{\sc He, K., Zhang, X., Ren, S., and Sun, J.}
\newblock Deep residual learning for image recognition.
\newblock {\em 2016 IEEE Conference on Computer Vision and Pattern Recognition
  (CVPR)\/} (Jun 2016).

\bibitem{He2}
{\sc He, K., Zhang, X., Ren, S., and Sun, J.}
\newblock Identity mappings in deep residual networks.
\newblock In {\em European conference on computer vision\/} (2016), Springer,
  pp.~630--645.

\bibitem{HH}
{\sc Higham, C.~F., and Higham, D.~J.}
\newblock Deep learning: An introduction for applied mathematicians.
\newblock https://arxiv.org/abs/1801.05894.

\bibitem{H}
{\sc Ho-Phuoc, T.}
\newblock {CIFAR10 to Compare Visual Recognition Performance between Deep
  Neural Networks and Humans}.
\newblock https://arxiv.org/abs/1811.07270.

\bibitem{IoffeSzeg}
{\sc Ioffe, S., and Szegedy, C.}
\newblock Batch normalization: Accelerating deep network training by reducing
  internal covariate shift.
\newblock {\em 32nd International Conference on Machine Learning, ICML 2015
  1\/} (2015), 448--456.

\bibitem{KH}
{\sc Krizhevsky, A., and Hinton, G.}
\newblock Learning multiple layers of features from tiny images.
\newblock Tech. rep., Technical report, University of Toronto, 2009.

\bibitem{LeBLi}
{\sc Le~Bris, C., and Lions, P.-L.}
\newblock Existence and uniqueness of solutions to {F}okker-{P}lanck type
  equations with irregular coefficients.
\newblock {\em Comm. Partial Differential Equations 33}, 7-9 (2008),
  1272--1317.

\bibitem{Lee}
{\sc Lee, C.-Y., Xie, S., Gallagher, P., Zhang, Z., and Tu, Z.}
\newblock Deeply-supervised nets.
\newblock {\em Journal of Machine Learning Research 38\/} (2015), 562--570.

\bibitem{LiLoss}
{\sc Li, H., Xu, Z., Taylor, G., Studer, C., and Goldstein, T.}
\newblock Visualizing the loss landscape of neural nets.
\newblock {\em Advances in Neural Information Processing Systems
  2018-December\/} (2018), 6389--6399.

\bibitem{LCW}
{\sc Li, Q., Chen, L., Tai, C., and Weinan, E.}
\newblock Maximum principle based algorithms for deep learning.
\newblock {\em Journal of Machine Learning Research 18\/} (2018), 1--29.

\bibitem{LTE}
{\sc Li, Q., Tai, C., and E, W.}
\newblock Stochastic modified equations and adaptive stochastic gradient
  algorithms.
\newblock In {\em Proceedings of the 34th International Conference on Machine
  Learning - Volume 70\/} (2017), ICML'17, JMLR.org, pp.~2101--2110.

\bibitem{LTE2}
{\sc Li, Q., Tai, C., and E, W.}
\newblock Stochastic modified equations and dynamics of stochastic gradient
  algorithms {I:} mathematical foundations.
\newblock {\em CoRR abs/1811.01558\/} (2018).

\bibitem{Lieb}
{\sc Lieberman, G.~M.}
\newblock {\em Second order parabolic differential equations}.
\newblock World Scientific Publishing Co., Inc., River Edge, NJ, 1996.

\bibitem{MGIVW}
{\sc Maddox, W., Garipov, T., Izmailov, P., Vetrov, D., and Wilson, A.~G.}
\newblock A simple baseline for {B}ayesian uncertainty in deep learning.
\newblock https://arxiv.org/abs/1902.02476.

\bibitem{MHB}
{\sc Mandt, S., Hoffman, M.~D., and Blei, D.~M.}
\newblock Stochastic gradient descent as approximate bayesian inference.
\newblock {\em J. Mach. Learn. Res. 18}, 1 (Jan. 2017), 4873--4907.

\bibitem{RZL}
{\sc Ramachandran, P., Zoph, B., and Le, Q.~V.}
\newblock Searching for activation functions, 2017.

\bibitem{R}
{\sc Royer, G.}
\newblock {\em An initiation to logarithmic {S}obolev inequalities}, vol.~14 of
  {\em SMF/AMS Texts and Monographs}.
\newblock American Mathematical Society, Providence, RI; Soci\'{e}t\'{e}
  Math\'{e}matique de France, Paris, 2007.
\newblock Translated from the 1999 French original by Donald Babbitt.

\bibitem{SafS}
{\sc Safran, I., and Shamir, O.}
\newblock Spurious local minima are common in two-layer {ReLU} neural networks.
\newblock {\em 35th International Conference on Machine Learning, ICML 2018
  10\/} (2018), 7031--7052.

\bibitem{S}
{\sc Shaposhnikov, S.}
\newblock On uniqueness of a probability solution to the cauchy problem for the
  {Fokker-Planck-Kolmogorov} equation.
\newblock {\em Theory of Probability and its Applications 56}, 1 (2012),
  96--115.

\bibitem{SJL}
{\sc Soltanolkotabi, M., Javanmard, A., and Lee, J.}
\newblock Theoretical insights into the optimization landscape of
  over-parameterized shallow neural networks.
\newblock {\em IEEE Transactions on Information Theory 65}, 2 (2019), 742--769.

\bibitem{Stroock}
{\sc Stroock, D.~W.}
\newblock {\em An introduction to the theory of large deviations}.
\newblock Universitext. Springer-Verlag, New York, 1984.

\bibitem{TVG}
{\sc Thorpe, M., and van Gennip, Y.}
\newblock Deep limits of residual neural networks.
\newblock https://arxiv.org/abs/1810.11741.

\bibitem{VBGS}
{\sc Vidal, R., Bruna, J., Giryes, R., and Soatto, S.}
\newblock Mathematics of deep learning.
\newblock https://arxiv.org/abs/1712.04741.

\bibitem{VBook2}
{\sc Villani, C.}
\newblock {\em Topics in optimal transportation}, vol.~58 of {\em Graduate
  Studies in Mathematics}.
\newblock American Mathematical Society, Providence, RI, 2003.

\bibitem{VBook}
{\sc Villani, C.}
\newblock {\em Optimal transport}, vol.~338 of {\em Grundlehren der
  Mathematischen Wissenschaften [Fundamental Principles of Mathematical
  Sciences]}.
\newblock Springer-Verlag, Berlin, 2009.
\newblock Old and new.

\bibitem{Z}
{\sc Zhou, D.-X.}
\newblock Universality of deep convolutional neural networks.
\newblock {\em Applied and Computational Harmonic Analysis\/} (2019).

\end{thebibliography}

\end{document}